\documentclass[letterpaper]{article}
\usepackage[margin=1in,dvips]{geometry}
\usepackage{graphicx,psfrag,amsmath,amsthm,amssymb}
\usepackage[numbers,sort&compress]{natbib}
\usepackage{url,color,booktabs}
\usepackage[bookmarksnumbered=true,bookmarksopen=true]{hyperref}

\definecolor{DarkRed}{rgb}{0.545,0,0}
\hypersetup{
  breaklinks=true, colorlinks=true, linkcolor=black,
  citecolor=black, urlcolor=DarkRed 
}

\def \ifempty#1{\def\temp{#1} \ifx\temp\empty }

\newcommand{\Q}{\ensuremath{\mathbf{Q}}}

\newcommand{\cc}{\ensuremath{\mathbf{c}}} 
\newcommand{\dd}{\ensuremath{\mathbf{d}}}

\newcommand{\h}{\ensuremath{\mathbf{h}}}

\renewcommand{\t}{\ensuremath{\mathbf{t}}}
\newcommand{\uu}{\ensuremath{\mathbf{u}}}

\newcommand{\w}{\ensuremath{\mathbf{w}}}
\newcommand{\x}{\ensuremath{\mathbf{x}}}

\newcommand{\0}{\ensuremath{\mathbf{0}}}
\newcommand{\1}{\ensuremath{\mathbf{1}}}

\newcommand{\bdelta}{\ensuremath{\boldsymbol{\delta}}}

\newcommand{\bbR}{\ensuremath{\mathbb{R}}}

\newcommand{\calD}{\ensuremath{\mathcal{D}}}

\newcommand{\calF}{\ensuremath{\mathcal{F}}}

\newcommand{\calL}{\ensuremath{\mathcal{L}}}

\newcommand{\calN}{\ensuremath{\mathcal{N}}}

\newcommand{\calP}{\ensuremath{\mathcal{P}}}

\newcommand{\calY}{\ensuremath{\mathcal{Y}}}

\newcommand{\abs}[2][]{%
  \ifempty{#1} {\left\lvert#2\right\rvert} \else {#1\lvert#2#1\rvert} \fi}
\newcommand{\norm}[2][]{%
  \ifempty{#1} {\left\lVert#2\right\rVert} \else {#1\lVert#2#1\rVert} \fi}

\newcommand{\caja}[4][1]{{%
    \renewcommand{\arraystretch}{#1}%
    \begin{tabular}[#2]{@{}#3@{}}%
      #4%
    \end{tabular}%
    }}

{%
\begin{list}{#1}{
\vspace{-\topsep}
\vspace{-\partopsep}
\setlength{\itemindent}{0cm}
\setlength{\rightmargin}{0cm}
\setlength{\listparindent}{0cm}
\settowidth{\labelwidth}{#1}
\setlength{\leftmargin}{\labelwidth}
\addtolength{\leftmargin}{\labelsep}
\setlength{\itemsep}{0cm}
}%
}%
{%
\end{list}
\vspace{-\topsep}
\vspace{-\partopsep}
}

{\begin{enumerate}%
}%
{\end{enumerate}}

\theoremstyle{plain}
\newtheorem{thm}{Theorem}[section]

\newtheorem*{lemma*}{Lemma}

\newtheorem*{prop*}{Proposition}
\newtheorem{cor}[thm]{Corollary}

\theoremstyle{definition}

\newtheorem*{defn*}{Definition}

\newtheorem*{exmp*}{Example}

\newtheorem*{conj*}{Conjecture}

\theoremstyle{remark}

\newtheorem*{rmk*}{Remark}

\graphicspath{{grf/},{grf-miguel/}}

\AtBeginDvi{}

\title{Counterfactual Explanations for Oblique Decision Trees: \\ Exact, Efficient Algorithms\thanks{A shorter version of this paper appears as \cite{CarreirHada21a}. Code implementing the algorithms is available from the authors' web page.}}
\author{
  Miguel \'A.\ Carreira-Perpi\~n\'an\hspace{5ex} Suryabhan Singh Hada \\
  Dept.\ of Computer Science \& Engineering, University of California, Merced \\
  {\url{http://eecs.ucmerced.edu}}
}
\date{March 1, 2021}

\begin{document}

\maketitle

\begin{abstract}

  We consider counterfactual explanations, the problem of minimally adjusting features in a source input instance so that it is classified as a target class under a given classifier. This has become a topic of recent interest as a way to query a trained model and suggest possible actions to overturn its decision. Mathematically, the problem is formally equivalent to that of finding adversarial examples, which also has attracted significant attention recently. Most work on either counterfactual explanations or adversarial examples has focused on differentiable classifiers, such as neural nets. We focus on classification trees, both axis-aligned and oblique (having hyperplane splits). Although here the counterfactual optimization problem is nonconvex and nondifferentiable, we show that an exact solution can be computed very efficiently, even with high-dimensional feature vectors and with both continuous and categorical features, and demonstrate it in different datasets and settings. The results are particularly relevant for finance, medicine or legal applications, where interpretability and counterfactual explanations are particularly important.

\end{abstract}

\section{Introduction}
\label{s:intro}

Practical deployment of deep learning and machine learning models has become widespread in the last decade, and there is enormous societal interest in AI as a technology that can provide intelligent, automated processing of tasks that up to now were hard for machines. At the same time, some concerns about AI systems (ethical, safety and others) have arisen as well. One is the problem of interpretability, i.e., explaining the functionality of an automated system. This is an old problem, which has been studied (possibly under different names, such as explainability or transparency) since decades ago in statistics and machine learning (e.g.\ \cite{Breiman_84a,Freitas14a,Guidot_18a,Lipton18a}). A second problem is explaining why the system made a decision, and how to contest it or---of particular interest here---how to change it. This problem is more recent and has become more pressing and widely known with the advent of legislation, such as the European Union's General Data Protection Regulation (GDPR) \cite{GoodmanFlaxman17a,Wachter_18a}, which requires AI systems to explain in some way its decisions to humans. That said, related problems have been studied in data mining or knowledge discovery from databases \cite{Yang_06c,Bella_11a,MartenProvos14a,Cui_15a}, in particular in applications such as customer relationship management (CRM).

In this paper we focus on a specific version of the second problem that, following \citet{Wachter_18a}, we will call a \emph{counterfactual explanation}, which refers to the fact that an event did not actually happen. For example, ``You were denied a loan because your annual income was \$30,000. If your income had been \$45,000, you would have been offered a loan.'' The second statement, or counterfactual, offers an alternative event that would result in the desired outcome (loan approval). Formally, a counterfactual explanation seeks the minimal change to a given feature vector that will change a classifier's decision in a prescribed way, and we will make this more precise as an optimization problem later. Compared to the problem of explaining the functionality of an automated system, counterfactual explanations are easier in that they do not require interpretability, as long as the explanations help a subject act rather than understand \cite{Wachter_18a}.

Formally, the problem of finding a counterfactual explanation is the same as that of finding an adversarial example \cite{Szeged_14a,Simony_14a,ZeilerFergus14a,Goodfel_15a}. The difference is in the underlying motivation. In a counterfactual explanation, one typically seeks a change of the source feature vector that is as small as possible (because changing features is seen as costly) and changes the classifier outcome (to a prescribed and more desirable one). In an adversarial example, one typically seeks a change of the source feature vector that is as small as possible (so that it is hard to detect) and changes the classifier outcome (to trick it into predicting the wrong outcome). Our focus will be in the solution of such optimization problems, although we will use as running example the case of counterfactual explanations.

Recent work has explored various ways of defining counterfactual explanation problems and solving them for different classifiers, in particular neural nets. Here we consider classification trees%
\footnote{Throughout, we consider hard decision trees, where an instance goes either left or right at each decision node, and hence it reaches a single leaf. We do not consider soft decision trees (such as hierarchical mixtures of experts; \cite{JordanJacobs94a}), where the instance traverses all root-leaf paths, each weighted by a certain probability.},
which are of particular interest compared to other models for several reasons. Trees are widely used in practice, can handle continuous and discrete features, are extremely fast for inference, can model nonlinear boundaries, and provide multiclass models naturally (without the need of constructions such as one-vs-all). Perhaps most importantly, trees are generally considered among the most interpretable models (certainly much more so that neural nets or forests). And, while (axis-aligned) trees have traditionally not been competitive in terms of predictive accuracy with other models, a recent algorithm (tree alternating optimization, TAO; \cite{Carreir21a,CarreirTavall18a,Zharmag_20a}) is able to learn oblique trees whose accuracy is much higher, which makes such trees competitive with other models.

Mathematically, trees provide a nonlinear classifier based on hierarchical, discontinuous splits, so the corresponding counterfactual problem is nonconvex and nondifferentiable. Yet, we show it can be solved exactly and efficiently, even if categorical variables exist. Our algorithm is very fast and suitable for interactive exploration of counterfactual explanations under different objectives or constraints. It can also provide, in a natural way, not just one but a diverse set of counterfactual explanations, which provide a range of ways in which the desired classifier outcome may be achieved.

\section{Related work}
\label{s:related}

Counterfactual explanations can be seen as a form of knowledge extraction from a trained machine learning model. This is the traditional realm of data mining, particularly in business and marketing \cite{Hand_01a,Han_11a,Aggarw15a,Witten_16a}. However, the precise formulation of counterfactual explanations as optimization problems given a classifier, source instance and target class (such as the one we follow in section~\ref{s:algorithm:basic-counterfactual}) and the various works exploring this research topic are quite recent. The formulation of counterfactual explanations can take different forms but always involve a distance function to measure how costly it is to change features (attributes) in a source instance, as well as a constraint or penalty term that ensures a target class is predicted. Optimizing a tradeoff of both of these yields the counterfactual instance. Most algorithms to solve the optimization assume differentiability of the classifier with respect to its input instance, so that gradient-based optimization can be applied. This has been particularly exploited with deep nets for adversarial examples and model inversion \cite{Szeged_14a,Simony_14a,ZeilerFergus14a,Goodfel_15a,MahendVedald16a,DosovitBrox16a,HadaCarreir19a}, two problems that are very similar to counterfactual explanations. Other methods are specific for linear models \cite{Ustun_19a,Russel19a,CarreirHada21c}. However, none of these algorithms apply to decision trees, which define nondifferentiable classifiers.

Some agnostic algorithms have been proposed which assume nothing about the classifier other than it can be evaluated on arbitrary instances, using some kind of random or approximate search \cite{Sharma_19a,Karimi_19a}. While these approaches are very general, they are computationally slow, particularly with high-dimensional instances, and give poor approximations to the optimal solution. One agnostic approach is to restrict the instance search space to a finite set of instances (such as the training set of the classifier), so the optimization involves a simple brute-force search, as in a database search. While this may be useful in some applications, it has a limited ability to explore the instance space, particularly if the problem constraints are satisfied by few instances, and is slow if the set has many, high-dimensional instances. A recent implementation of this approach is in the What-If Tool \cite{Wexler_20a} for interaction and visualization of machine learning systems.

Our paper is specifically about decision trees, both axis-aligned and oblique, for multiclass classification and using continuous and categorical features. What little work exists in counterfactual explanations research about decision trees has focused on axis-aligned trees (or forests) for binary classification only, as far as we know. \citet{Yang_06c}, motivated by customer relationship management, seek to infer actions from a binary classification tree (attrition vs no attrition), specifically to move a group of instances (customers) from some source leaves to some target leaves of the tree. The problem is different from a standard counterfactual explanation and is restricted to categorical features only. It takes the form of a maximum coverage problem, which is NP-complete and is approximated with a greedy algorithm. \citet{Bella_11a} consider a restricted form of counterfactual explanation over a single, ``negotiable'' feature, which must be continuous and satisfy certain sensitivity and monotonicity conditions. \citet{Cui_15a} formulate a type of counterfactual problem for binary classification forests, show it is NP-hard, and encode it as an integer linear program, which can be (approximately) solved by existing solvers. It is practical only for low-dimensional feature vectors, and even then it takes seconds or minutes for one instance. \citet{Tolomei_19a} also consider a restricted form of counterfactual problem for a binary classification forest and propose an approximate algorithm, based on propagating the source instance down each tree towards a leaf. This is claimed to be optimal if the forest contains a single tree. However, as our experiments show, this is not true.

\section{Counterfactual explanation for oblique trees: an exact algorithm}
\label{s:algorithm}

\subsection{Definitions}
\label{s:algorithm:def}

Assume we are given a classification tree that can map an input instance $\x \in \bbR^D$, with $D$ real features (attributes), to a class in $\{1,\dots,K\}$. Assume the tree is rooted, directed and binary (where each decision node has two children) with decision nodes and leaves indexed in the sets \calD\ and \calL, respectively, and $\calN = \calD \cup \calL$. We index the root as $1 \in \calD$. For example, in fig.~\ref{f:T1} we have $\calN = \{1,\dots,17\}$, $\calL = \{8,10,\dots,17\}$ and $\calD = \calN \setminus \calL$. Each decision node $i \in \calD$ has a real-valued decision function $f_i(\x)$ such that input instance $\x \in \bbR^D$ is sent down $i$'s right child if $f_i(\x) \ge 0$ and down $i$'s left child otherwise. For oblique trees, the decision function is a hyperplane (linear combination of all the features) $f_i(\x) = \w^T_i \x + b_i$, with fixed weight vector $\w_i \in \bbR^D$ and bias $b_i \in \bbR$. For axis-aligned trees, $\w_i$ is an indicator vector (having one element equal to 1 and the rest equal to 0). Each leaf $i \in \calL$ is labeled with one class label $y_i \in \{1,\dots,K\}$. The class $T(\x) \in \{1,\dots,K\}$ predicted by the tree for an input instance \x\ is found by sending \x\ down, via the decision nodes, to exactly one leaf and outputting its label. The parameters $\{\w_i,b_i\}_{i \in \calD}$ and $\{y_i\}_{i \in \calL}$ are estimated by TAO \cite{Carreir21a,CarreirTavall18a} (or another algorithm) when learning the tree from a labeled training set.

The tree partitions the input space into $\abs{\calL}$ regions, one per leaf, as shown in figures~\ref{f:T1} and~\ref{f:T2} (right panels). Each region is an axis-aligned box (for axis-aligned trees) or polytope (for oblique trees) given by the intersection of the hyperplanes found in the path from the root to the leaf. Specifically, define a linear constraint $z_i (\w^T_i \x + b_i) \ge 0$ for decision node $i$ where $z_i = +1$ if going down its right child and $z_i = -1$ if going down its left child. Then we define the constraint vector for leaf $i \in \calL$ as $\h_i(\x) = ( z_j (\w^T_j \x + b_j) )_{j \in \calP_i \setminus \{i\}}$, where $\calP_i = \{1,\dots,i\}$ is the path of nodes from the root (node 1) to leaf $i$. We call $\calF_i = \{\x \in \bbR^D\mathpunct{:}\ \h_i(\x) \ge 0\}$ the corresponding feasible set, i.e., the region in input space of leaf $i$. For example, in fig.~\ref{f:T2} (left) the path from the root to leaf 15 is $\calP_{15} = \{1,3,6,10,15\}$ and its region is given by:
\begin{equation*}
  \h_{15}(\x) =
  \begin{pmatrix}
    \phantom{-}f_1(\x) \\ -f_3(\x) \\ -f_6(\x) \\ \phantom{-}f_{10}(\x)
  \end{pmatrix} =
  \begin{pmatrix}
    \phantom{-}\w^T_1 \x + b_1 \\ -\w^T_3 \x - b_3 \\ -\w^T_6 \x - b_6 \\ \phantom{-}\w^T_{10} \x + b_{10}
  \end{pmatrix} \ge \0.
\end{equation*}

\begin{figure*}[t!]
  \centering
  \begin{tabular}{@{}c@{\hspace{0.08\linewidth}}c@{}}
    \scriptsize
    \psfrag{1}[][b]{\raisebox{2ex}[0pt][0pt]{\makebox[0pt][r]{1\hspace{1.5ex}}}$f_1$}
    \psfrag{2}[][b]{\raisebox{2ex}[0pt][0pt]{\makebox[0pt][r]{2\hspace{1.5ex}}}$f_2$}
    \psfrag{3}[][b]{\raisebox{2ex}[0pt][0pt]{\makebox[0pt][r]{3\hspace{1.5ex}}}$f_3$}
    \psfrag{4}[][b]{\raisebox{2ex}[0pt][0pt]{\makebox[0pt][r]{4\hspace{1.5ex}}}$f_4$}
    \psfrag{5}[][b]{$f_5$\raisebox{2ex}[0pt][0pt]{\makebox[0pt][l]{\hspace{1.5ex}5}}}
    \psfrag{6}[][b]{\raisebox{2ex}[0pt][0pt]{\makebox[0pt][r]{6\hspace{1.5ex}}}$f_6$}
    \psfrag{7}[][b]{$f_7$\raisebox{2ex}[0pt][0pt]{\makebox[0pt][l]{\hspace{1.5ex}7}}}
    \psfrag{8}[][b]{\raisebox{2ex}[0pt][0pt]{\makebox[0pt][r]{8\hspace{1.5ex}}}\phantom{$f_8$}}
    \psfrag{9}[][b]{\raisebox{2ex}[0pt][0pt]{\makebox[0pt][r]{9\hspace{1.5ex}}}$f_9$}
    \psfrag{10}[][b]{\raisebox{2ex}[0pt][0pt]{\makebox[0pt][r]{10\hspace{1.5ex}}}\phantom{$f_{10}$}}
    \psfrag{11}[][b]{\raisebox{2ex}[0pt][0pt]{\makebox[0pt][r]{11\hspace{1.5ex}}}\phantom{$f_{11}$}}
    \psfrag{12}[][b]{\raisebox{2ex}[0pt][0pt]{\makebox[0pt][r]{12\hspace{1.5ex}}}\phantom{$f_{12}$}}
    \psfrag{13}[][b]{\raisebox{2ex}[0pt][0pt]{\makebox[0pt][r]{13\hspace{1.5ex}}}\phantom{$f_{13}$}}
    \psfrag{14}[][b]{\raisebox{2ex}[0pt][0pt]{\makebox[0pt][r]{14\hspace{1.5ex}}}\phantom{$f_{14}$}}
    \psfrag{15}[][b]{\raisebox{2ex}[0pt][0pt]{\makebox[0pt][r]{15\hspace{1.5ex}}}\phantom{$f_{15}$}}
    \psfrag{16}[][b]{\raisebox{2ex}[0pt][0pt]{\makebox[0pt][r]{16\hspace{1.5ex}}}\phantom{$f_{16}$}}
    \psfrag{17}[][b]{\raisebox{2ex}[0pt][0pt]{\makebox[0pt][r]{17\hspace{1.5ex}}}\phantom{$f_{17}$}}
    \includegraphics*[width=0.55\linewidth]{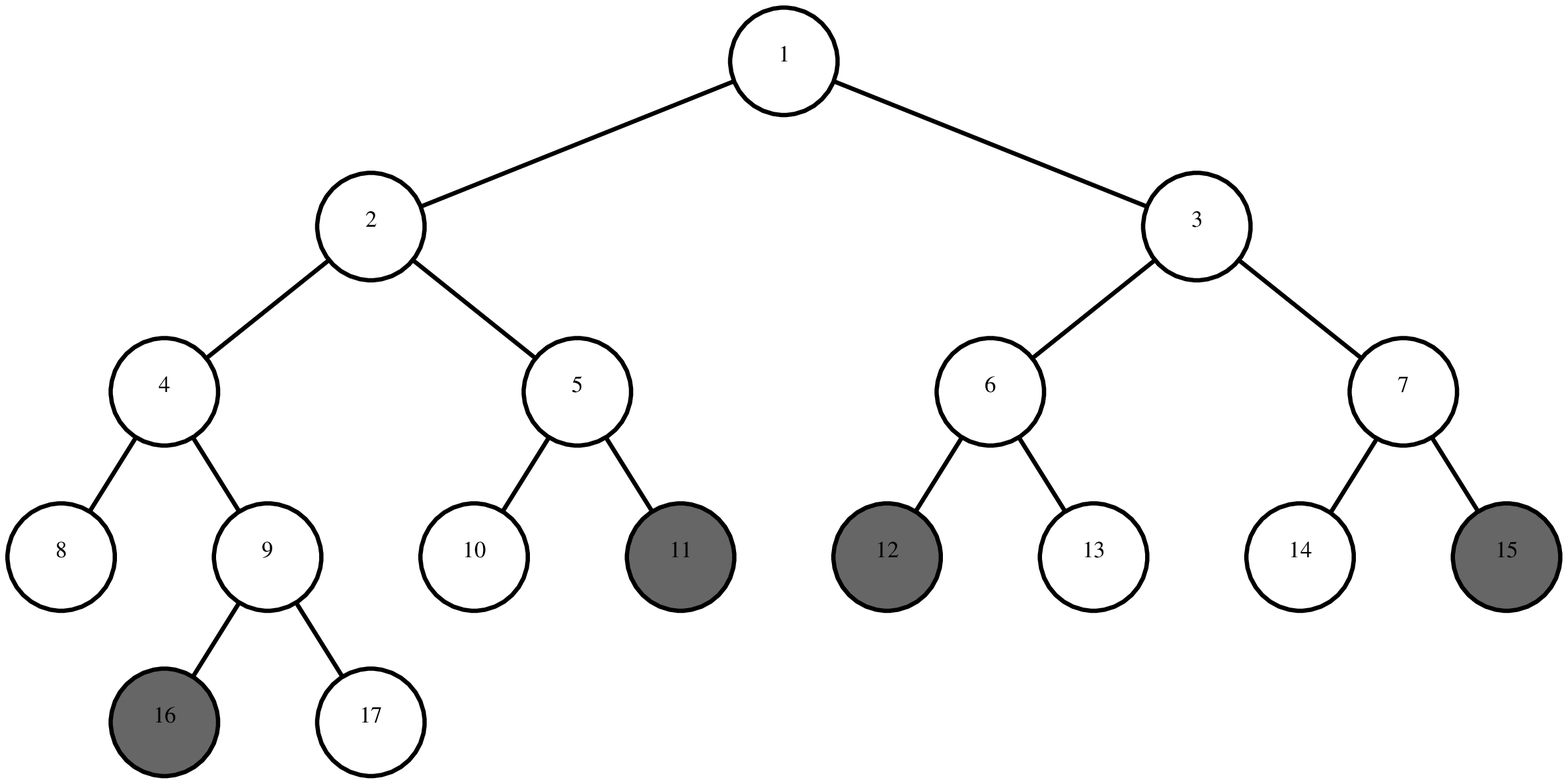} &
    \scriptsize
    \psfrag{v}[]{$\overline{\x}$}
    \psfrag{xx}[l]{$\x^*$}
    \psfrag{x1}{$x_1$}
    \psfrag{x2}{$x_2$}
    \psfrag{f1}[t]{\colorbox{white}{\scriptsize $f_1$}}
    \psfrag{f2}[r][Br]{\colorbox{white}{\scriptsize $f_2$}}
    \psfrag{f3}[l]{\colorbox{white}{\scriptsize $f_3$}}
    \psfrag{f4}[r][Br]{\colorbox{white}{\scriptsize $f_4$}}
    \psfrag{f5}[b]{\colorbox{white}{\scriptsize $f_5$}}
    \psfrag{f6}[t]{\colorbox{white}{\scriptsize $f_6$}}
    \psfrag{f7}[b]{\colorbox{white}{\scriptsize $f_7$}}
    \psfrag{f9}[]{\colorbox{white}{\scriptsize $f_9$}}
    \psfrag{8}[]{\fcolorbox{black}{white}{$8$}}
    \psfrag{10}[]{\fcolorbox{black}{white}{$10$}}
    \psfrag{11}[]{\fcolorbox{black}{white}{$11$}}
    \psfrag{12}[]{\fcolorbox{black}{white}{$12$}}
    \psfrag{13}[]{\fcolorbox{black}{white}{$13$}}
    \psfrag{14}[]{\fcolorbox{black}{white}{$14$}}
    \psfrag{15}[]{\fcolorbox{black}{white}{$15$}}
    \psfrag{16}[]{\fcolorbox{black}{white}{$16$}}
    \psfrag{17}[]{\fcolorbox{black}{white}{$17$}}
    \includegraphics*[width=0.37\linewidth]{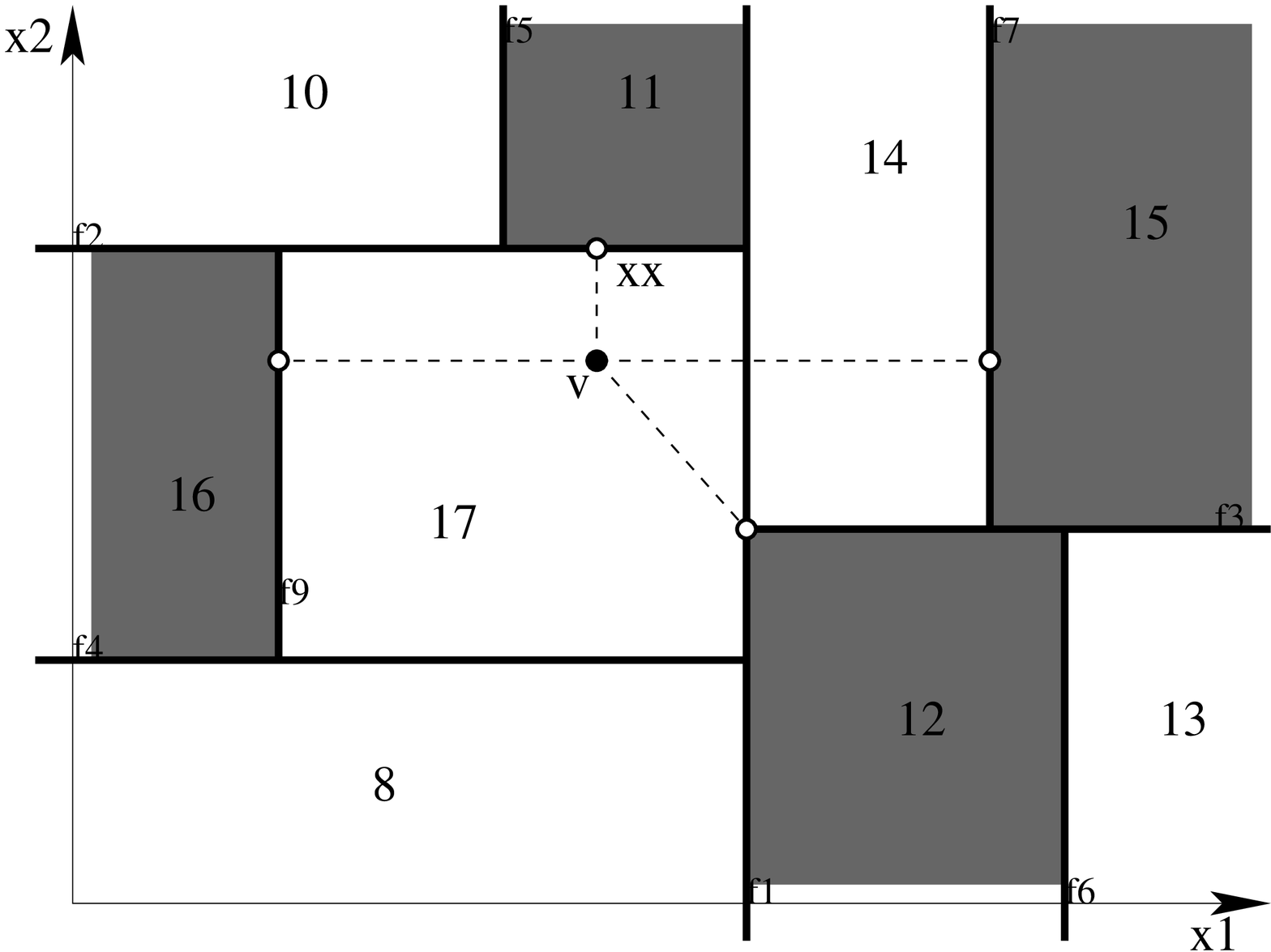}
  \end{tabular}
  \caption{\emph{Left}: an axis-aligned classification tree with $K = 2$ classes (colored white and gray). A decision node $i$ sends an input instance \x\ to its right child if $f_i(\x) \ge 0$ and to its left child otherwise. For example, for node 5 we have $f_5(\x) = x_1 + b_5$, i.e., it thresholds $x_1$ and hence creates a vertical split. Likewise, for node 4 we have $f_4(\x) = x_2 + b_4$, i.e., it thresholds $x_2$ and hence creates a horizontal split. Each leaf $i$ is colored according to its class label $y_i \in \{1,\dots,K\}$. \emph{Right}: the space of the input instances $\x \in \bbR^2$, assumed two-dimensional, partitioned according to each leaf's region, which is an axis-aligned box (the region boundaries are labeled with the corresponding decision node function). The source instance is $\overline{\x}$, of the white class. The counterfactual instance (using the squared $\ell_2$ distance, $\norm{\x - \overline{\x}}^2$) subject to changing to the gray class is $\x^*$, which is closest to $\overline{\x}$. We also show the closest instances to $\overline{\x}$ within each leaf of the gray class.}
  \label{f:T1}
\end{figure*}

\begin{figure*}[t!]
  \centering
  \begin{tabular}{@{}c@{\hspace{0.08\linewidth}}c@{}}
    \scriptsize
    \psfrag{1}[][b]{\raisebox{2ex}[0pt][0pt]{\makebox[0pt][r]{1\hspace{1.5ex}}}$f_1$}
    \psfrag{2}[][b]{\raisebox{2ex}[0pt][0pt]{\makebox[0pt][r]{2\hspace{1.5ex}}}$f_2$}
    \psfrag{3}[][b]{\raisebox{2ex}[0pt][0pt]{\makebox[0pt][r]{3\hspace{1.5ex}}}$f_3$}
    \psfrag{4}[][b]{\raisebox{2ex}[0pt][0pt]{\makebox[0pt][r]{4\hspace{1.5ex}}}$f_4$}
    \psfrag{5}[][b]{\phantom{$f_5$}\raisebox{2ex}[0pt][0pt]{\makebox[0pt][l]{\hspace{1.5ex}5}}}
    \psfrag{6}[][b]{\raisebox{2ex}[0pt][0pt]{\makebox[0pt][r]{6\hspace{1.5ex}}}$f_6$}
    \psfrag{7}[][b]{$f_7$\raisebox{2ex}[0pt][0pt]{\makebox[0pt][l]{\hspace{1.5ex}7}}}
    \psfrag{8}[][b]{\raisebox{2ex}[0pt][0pt]{\makebox[0pt][r]{8\hspace{1.5ex}}}\phantom{$f_8$}}
    \psfrag{9}[][b]{\raisebox{2ex}[0pt][0pt]{\makebox[0pt][r]{9\hspace{1.5ex}}}\phantom{$f_9$}}
    \psfrag{10}[][b]{\raisebox{2ex}[0pt][0pt]{\makebox[0pt][r]{10\hspace{1.5ex}}}$f_{10}$}
    \psfrag{11}[][b]{\raisebox{2ex}[0pt][0pt]{\makebox[0pt][r]{11\hspace{1.5ex}}}\phantom{$f_{11}$}}
    \psfrag{12}[][b]{\raisebox{2ex}[0pt][0pt]{\makebox[0pt][r]{12\hspace{1.5ex}}}\phantom{$f_{12}$}}
    \psfrag{13}[][b]{\raisebox{2ex}[0pt][0pt]{\makebox[0pt][r]{13\hspace{1.5ex}}}\phantom{$f_{13}$}}
    \psfrag{14}[][b]{\raisebox{2ex}[0pt][0pt]{\makebox[0pt][r]{14\hspace{1.5ex}}}\phantom{$f_{14}$}}
    \psfrag{15}[][b]{\raisebox{2ex}[0pt][0pt]{\makebox[0pt][r]{15\hspace{1.5ex}}}\phantom{$f_{15}$}}
    \psfrag{16}[][b]{\raisebox{2ex}[0pt][0pt]{\makebox[0pt][r]{16\hspace{1.5ex}}}\phantom{$f_{16}$}}
    \psfrag{17}[][b]{\raisebox{2ex}[0pt][0pt]{\makebox[0pt][r]{17\hspace{1.5ex}}}\phantom{$f_{17}$}}
    \includegraphics*[width=0.55\linewidth]{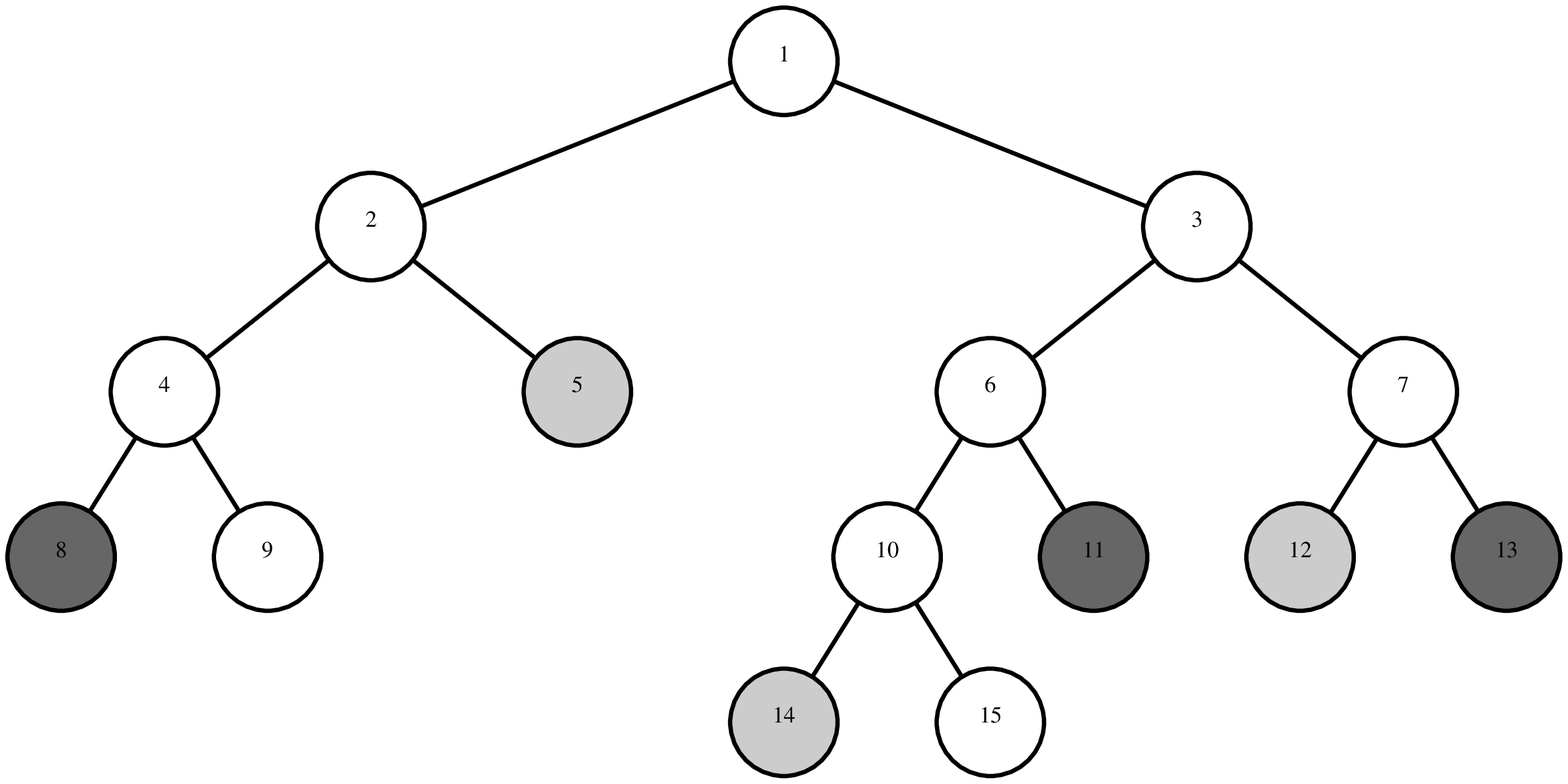} &
    \small
    \psfrag{v}[]{$\overline{\x}$}
    \psfrag{xx}{$\x^*$}
    \psfrag{x1}{$x_1$}
    \psfrag{x2}{$x_2$}
    \psfrag{f1}[t]{\colorbox{white}{\scriptsize $f_1$}}
    \psfrag{f2}[r][Br]{\colorbox{white}{\scriptsize $f_2$}}
    \psfrag{f3}[bl]{\colorbox{white}{\scriptsize $f_3$}}
    \psfrag{f4}[tr][Br]{\colorbox{white}{\scriptsize $f_4$}}
    \psfrag{f5}[]{\colorbox{white}{\scriptsize $f_5$}}
    \psfrag{f6}[t]{\colorbox{white}{\scriptsize $f_6$}}
    \psfrag{f7}[]{\colorbox{white}{\scriptsize $f_7$}}
    \psfrag{f9}[]{\colorbox{white}{\scriptsize $f_9$}}
    \psfrag{5}[]{\fcolorbox{black}{white}{$5$}}
    \psfrag{8}[]{\fcolorbox{black}{white}{$8$}}
    \psfrag{9}[]{\fcolorbox{black}{white}{$9$}}
    \psfrag{f10}[]{\colorbox{white}{\scriptsize $f_{10}$}}
    \psfrag{11}[]{\fcolorbox{black}{white}{$11$}}
    \psfrag{12}[]{\fcolorbox{black}{white}{$12$}}
    \psfrag{13}[]{\fcolorbox{black}{white}{$13$}}
    \psfrag{14}[]{\fcolorbox{black}{white}{$14$}}
    \psfrag{15}[]{\fcolorbox{black}{white}{$15$}}
    \psfrag{16}[]{\fcolorbox{black}{white}{$16$}}
    \psfrag{17}[]{\fcolorbox{black}{white}{$17$}}
    \includegraphics*[width=0.37\linewidth]{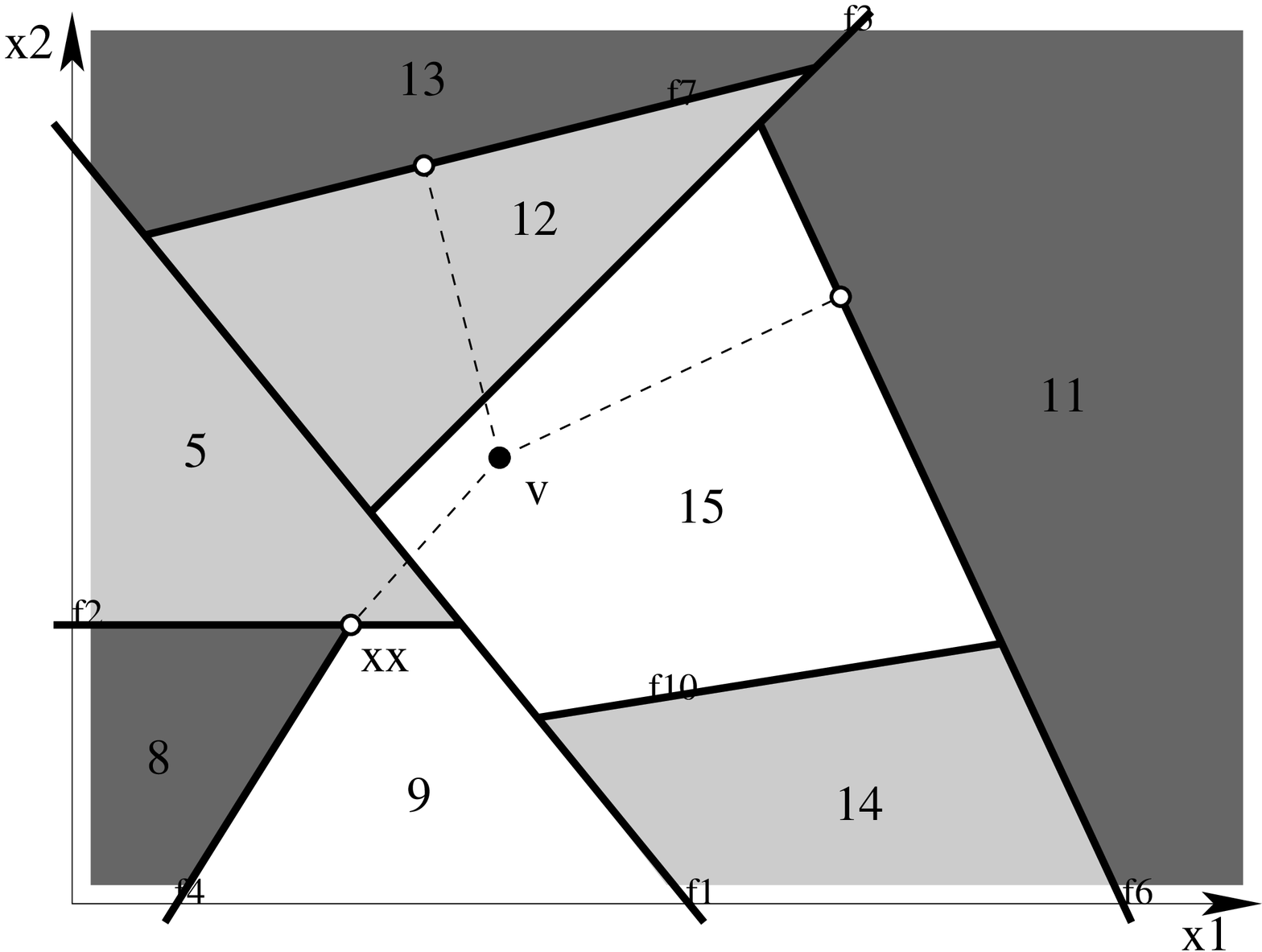}
  \end{tabular}
  \caption{\emph{Left}: like fig.~\ref{f:T1} but for an oblique classification tree with $K = 3$ classes (colored white, light gray and gray). Unlike in an axis-aligned tree, where each decision function uses a single feature, in the oblique tree it uses a linear combination of them: $f_i(\x) = \w^T_i \x + b_i$. The source instance $\overline{\x}$ is in the white class and the counterfactual one (using the $\ell_2$ distance) subject to being in the gray class is $\x^*$.}
  \label{f:T2}
\end{figure*}

\subsection{Learning axis-aligned and oblique trees: Tree Alternating Optimization (TAO)}
\label{s:TAO}

Traditionally, classification trees have been learned from a labeled training set using greedy top-down algorithms that split an initial, root node (using a class purity criterion) and proceed recursively with its children until a stopping criterion is achieved. The classical examples are CART \cite{Breiman_84a} and C4.5 \cite{Quinlan93a}. However, these algorithms achieve quite suboptimal trees, particularly if they are applied to oblique trees (having hyperplane decision nodes). Still, they are widely use in practice to learn axis-aligned (univariate) trees.

Recently, a new algorithm has been proposed that is able to train trees more accurately, both axis-aligned and oblique: Tree Alternating Optimization (TAO) \cite{Carreir21a,CarreirTavall18a}. TAO does not grow a tree greedily. Instead, it takes a given tree structure with initial parameter values at the nodes and optimizes a loss function over these parameters---much as one would train, say, a neural net, except the tree is not differentiable. TAO essentially works by optimizing the parameters of each node (decision node or leaf) at a time. Each iteration of TAO is guaranteed to reduce or leave unchanged the classification error, which results in trees that are smaller yet much more accurate than those trained with CART, C4.5 or other algorithms, as shown in a range of datasets \cite{Zharmag_20a}. Furthermore, the predictive accuracy of oblique trees trained with TAO becomes comparable to that of state-of-the-art classifiers. Ensembles of TAO trees also improve considerably over traditional forests \cite{CarreirZharmag20a,ZharmagCarreir20a}. Other types of trees can also be learned \cite{ZharmagCarreir21a}. More details about TAO can be found in \cite{Carreir21a,Zharmag_20a}.

In this paper, we illustrate our counterfactual explanation algorithm with axis-aligned trees trained with CART and oblique trees trained with TAO.

\subsection{Basic counterfactual optimization problem}
\label{s:algorithm:basic-counterfactual}

We start by giving the simplest, but also most important, formulation of finding an optimal counterfactual explanation. Assume we are given a \emph{source input instance} $\overline{\x} \in \bbR^D$ which is classified by the tree as class $\overline{y}$, i.e., $T(\overline{\x}) = \overline{y}$, and we want to find the closest instance $\x^*$ that would be classified as another class $y \neq \overline{y}$ (the \emph{target class})%
\footnote{We can also consider a formulation of the problem where the target class $y$ is the same as the source class $\overline{y}$, but we seek an instance \x\ having a lower cost than the source instance $\overline{\x}$. For example, even if a subject is classified as approved for a mortgage, it may be possible to reduce the initial payment and still be approved. This requires redefining the cost $E(\x)$ accordingly.}.
We define the \emph{counterfactual explanation} for $\overline{\x}$ as the (or a) minimizer $\x^*$ of the following problem:
\begin{equation}
  \label{e:basic-counterfactual}
  \min_{\x \in \bbR^D}{ E(\x;\overline{\x}) } \quad \text{s.t.} \quad T(\x) = y,\ \cc(\x) = \0,\ \dd(\x) \ge \0
\end{equation}
where $E(\x;\overline{\x})$ is a cost of changing features of $\overline{\x}$, and $\cc(\x)$ and $\dd(\x)$ are equality and inequality constraints (in vector form), all of which will be defined more precisely in sections~\ref{s:algorithm:E} and~\ref{s:algorithm:C}. The fundamental idea is that problem~\eqref{e:basic-counterfactual} seeks an instance \x\ that is as close as possible to $\overline{\x}$ while being classified as class $y$ by the tree and satisfying the constraints $\cc(\x)$ and $\dd(\x)$.

The constraint $T(\x) = y$ makes the problem severely nonconvex, nonlinear and nondifferentiable because of the tree function $T(\x)$. However, the following simple observation, whose proof is obvious, shows that we can solve the problem exactly and efficiently.
\begin{thm}
  \label{th:tree}
  Problem~\eqref{e:basic-counterfactual} is equivalent to:
  \begin{equation}
    \label{e:basic-counterfactual2}
    \min_{i \in \calL}{ \min_{\x \in \bbR^D}{ E(\x;\overline{\x}) } } \quad \text{s.t.} \quad y_i = y,\ \h_i(\x) \ge \0,\  \cc(\x) = \0,\ \dd(\x) \ge \0.
  \end{equation}
\end{thm}
In English, what this means is that solving problem~\eqref{e:basic-counterfactual} over the entire space can be done by solving it within each leaf's region and then picking the leaf with the best solution. This is shown in figures~\ref{f:T1} and~\ref{f:T2} (right panels). That is, the problem has the form of a mixed-integer optimization where the integer part is done by enumeration (over the leaves) and the continuous part (within each leaf) by other means to be described later. This is true for any tree as long as it has hard decisions at the decision nodes, even if the decision functions and leaf predictors are more complex than the hyperplanes and constant labels, respectively, that we consider here. Since the number of leaves in a tree is relatively small, this is computationally possible (in particular, oblique trees have very few leaves compared to axis-aligned ones).

Hence, the problem we still need to solve is the problem over a single leaf $i \in \calL$ (having the desired label $y_i = y$), and henceforth we focus on this. We write it as:
\begin{equation}
  \label{e:basic-counterfactual-leaf}
  \min_{\x \in \bbR^D}{ E(\x;\overline{\x}) } \quad \text{s.t.}  \quad\h_i(\x) \ge \0,\ \cc(\x) = \0,\ \dd(\x) \ge \0.
\hspace{-.4em}
\end{equation}
If the function $E(\x;\cdot)$ is convex over \x\ and the constraints $\cc(\x)$ and $\dd(\x)$ are linear, then this problem is convex (since for oblique trees $\h_i(\x)$ is linear). In particular, if $E$ is linear or quadratic then the problem is a linear program (LP) or a convex quadratic program (QP), both of which can be solved very efficiently with existing solvers, more so because the number of variables $D$ is usually not very large in practice (thousands at most, and in some important applications it can be very small).

\subsection{Separable problems: axis-aligned trees}
\label{s:algorithm:separable}

The following result, whose proof is immediate, vastly simplifies the problem for axis-aligned trees.
\begin{thm}
  \label{th:separable}
  In problem~\eqref{e:basic-counterfactual}, assume that each constraint depends on a single element of \x\ (not necessarily the same) and that the objective function is separable, i.e., $E(\x;\overline{\x}) = \sum^D_{d=1}{ E_d(x_d;\overline{x}_d) }$. Then the problem separates over the variables $x_1,\dots,x_D$.
\end{thm}
This means that, within each leaf, we can solve for each $x_d$ independently, by minimizing $E_d(x_d;\overline{x}_d)$ subject to the constraints on $x_d$. Further, the solution is given by the following result.
\begin{thm}
  \label{th:scalar}
  Consider the scalar constrained optimization problem, where the bounds can take the values $l_d = -\infty$ and $u_d = \infty$:
  \begin{equation}
    \label{e:scalar}
    \min_{x_d \in \bbR}{ E_d(x_d;\overline{x}_d) } \quad \text{s.t.} \quad l_d \le x_d \le u_d.
  \end{equation}
  Assume $E_d$ is convex on $x_d$ and satisfies $E_d(\overline{x}_d;\overline{x}_d) = 0$ and $E_d(x_d;\overline{x}_d) \ge 0$ $\forall x_d \in \bbR$. Then $x^*_d$, defined as the median of $\overline{x}_d$, $l_d$ and $u_d$, is a global minimizer of the problem:
  \begin{equation}
    \label{e:median}
    x^*_d = \text{median}(\overline{x}_d,l_d,u_d) = \begin{cases} l_d, & \overline{x}_d < l_d \\ u_d, & \overline{x}_d > u_d \\ \overline{x}_d, & \text{otherwise} \end{cases}.
  \end{equation}
\end{thm}
\begin{proof}
  From the assumption over $E_d$ we have that $\overline{x}_d$ is a global minimizer of $E_d$. The result follows by comparing $\overline{x}_d$ with $l_d$ and $u_d$; see fig.~\ref{f:median}.
\end{proof}
The previous theorem does not consider equality constraints because, in a scalar problem, they trivially provide the solution (an equality constraint ``$x_d =$ value'' implies $x^*_d =$ value). The inequalities ``$l_d \le x_d \le u_d$'' in the theorem are obtained by collecting all the inequalities in the problem~\eqref{e:basic-counterfactual} that involve $x_d$.

Importantly, these theorems apply to axis-aligned trees (assuming each of the extra constraints $\cc(\x)$ and $\dd(\x)$ depends individually on a single feature), because each of the constraints $\h_i(\x) \ge \0$ in the path from the root to leaf $i$ involve a single feature of \x. This makes solving the counterfactual explanation problem exceedingly fast for axis-aligned trees. We can represent each leaf $i \in \calL$ by a bounding box $\h_i \le \x \le \uu_i$ (which collects the constraints along the path from the root to $i$), solve elementwise by applying the median formula above within each leaf, and finally return the result of the best leaf.

\begin{figure}[t]
  \centering
  \psfrag{z}[][]{$x_d$}
  \psfrag{Mz}[][][1][-90]{$x^*_d$}
  \psfrag{0}[B][B]{\raisebox{-5pt}{$l_d$}}
  \psfrag{1}[B][B]{\raisebox{-5pt}{$u_d$}}
  \psfrag{00}[B][B]{$l_d$}
  \psfrag{01}[B][B]{$u_d$}
  \begin{tabular}{@{}c@{\hspace{0.01\linewidth}}c@{\hspace{0.01\linewidth}}c@{\hspace{0.03\linewidth}}c@{}}
    \includegraphics[width=0.23\linewidth]{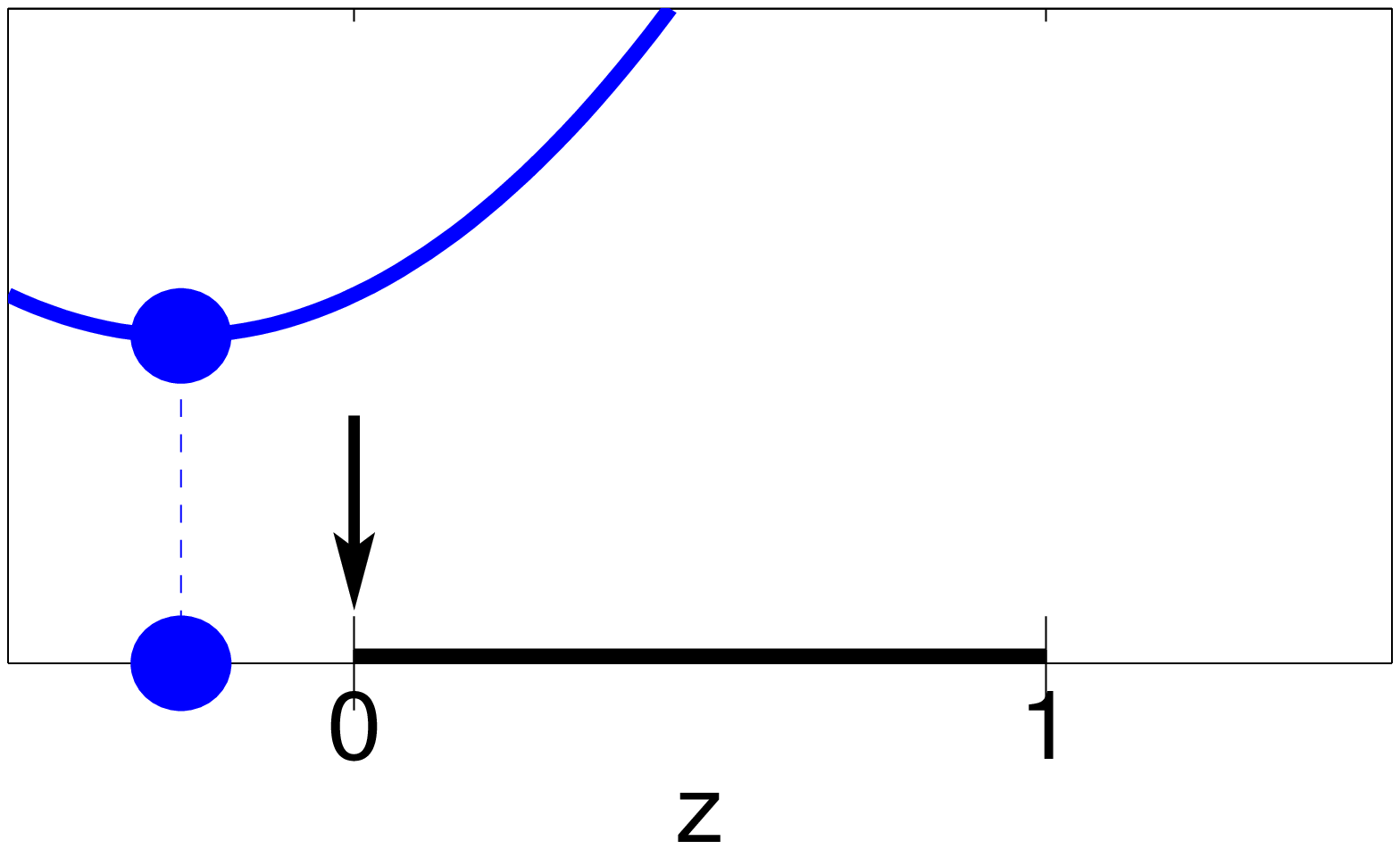} &
    \includegraphics[width=0.23\linewidth]{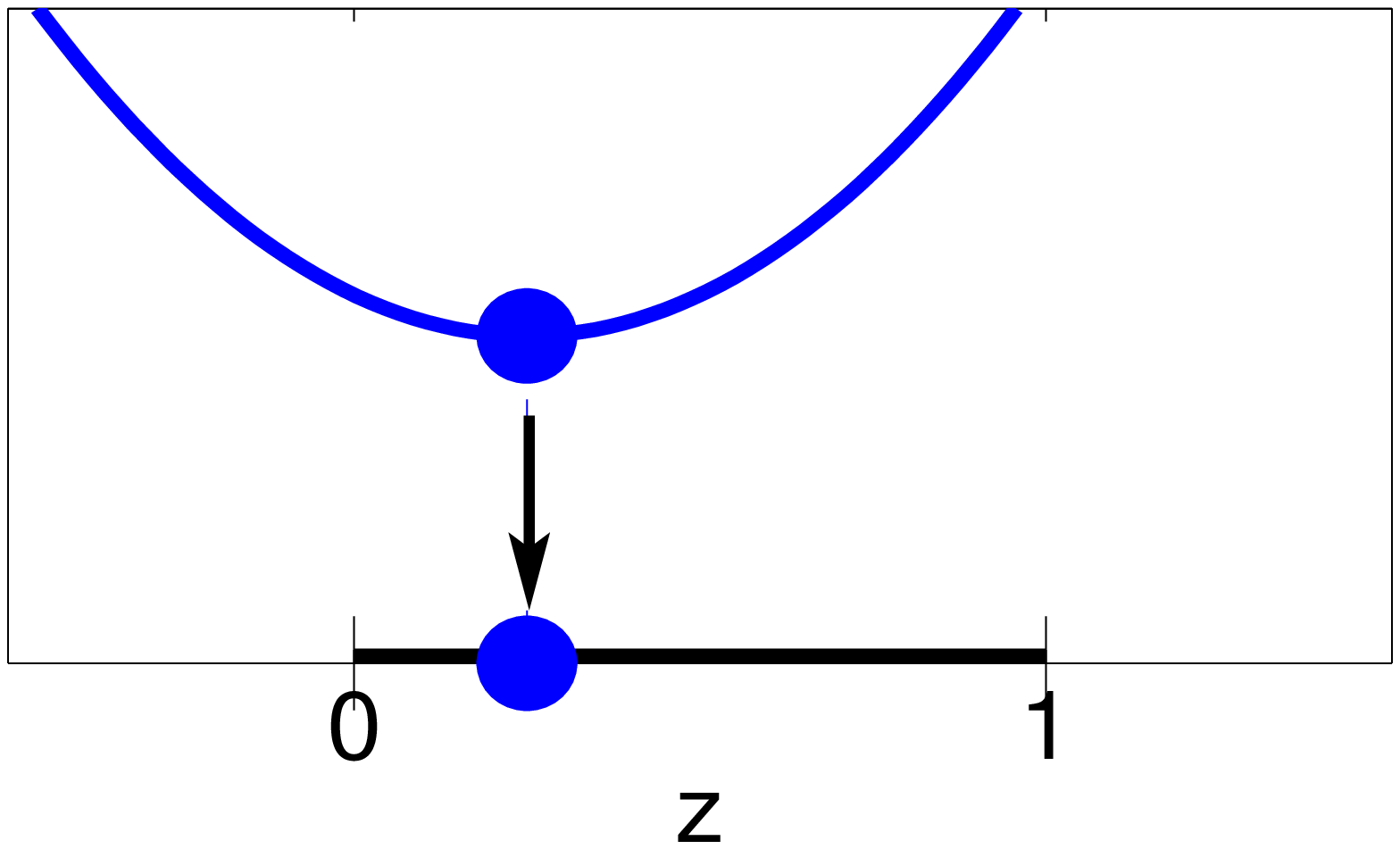} &
    \includegraphics[width=0.23\linewidth]{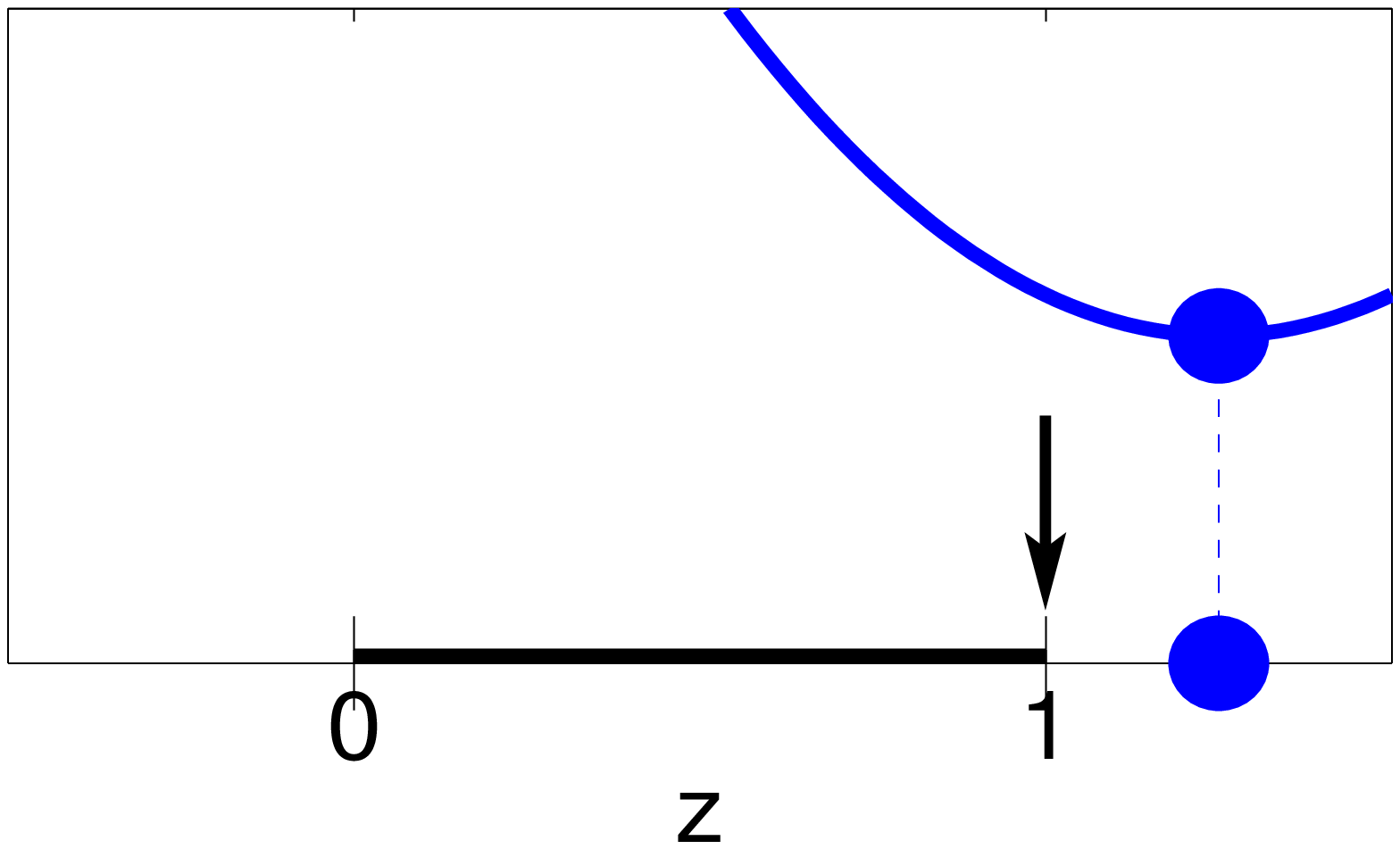} &
    \includegraphics[width=0.25\linewidth]{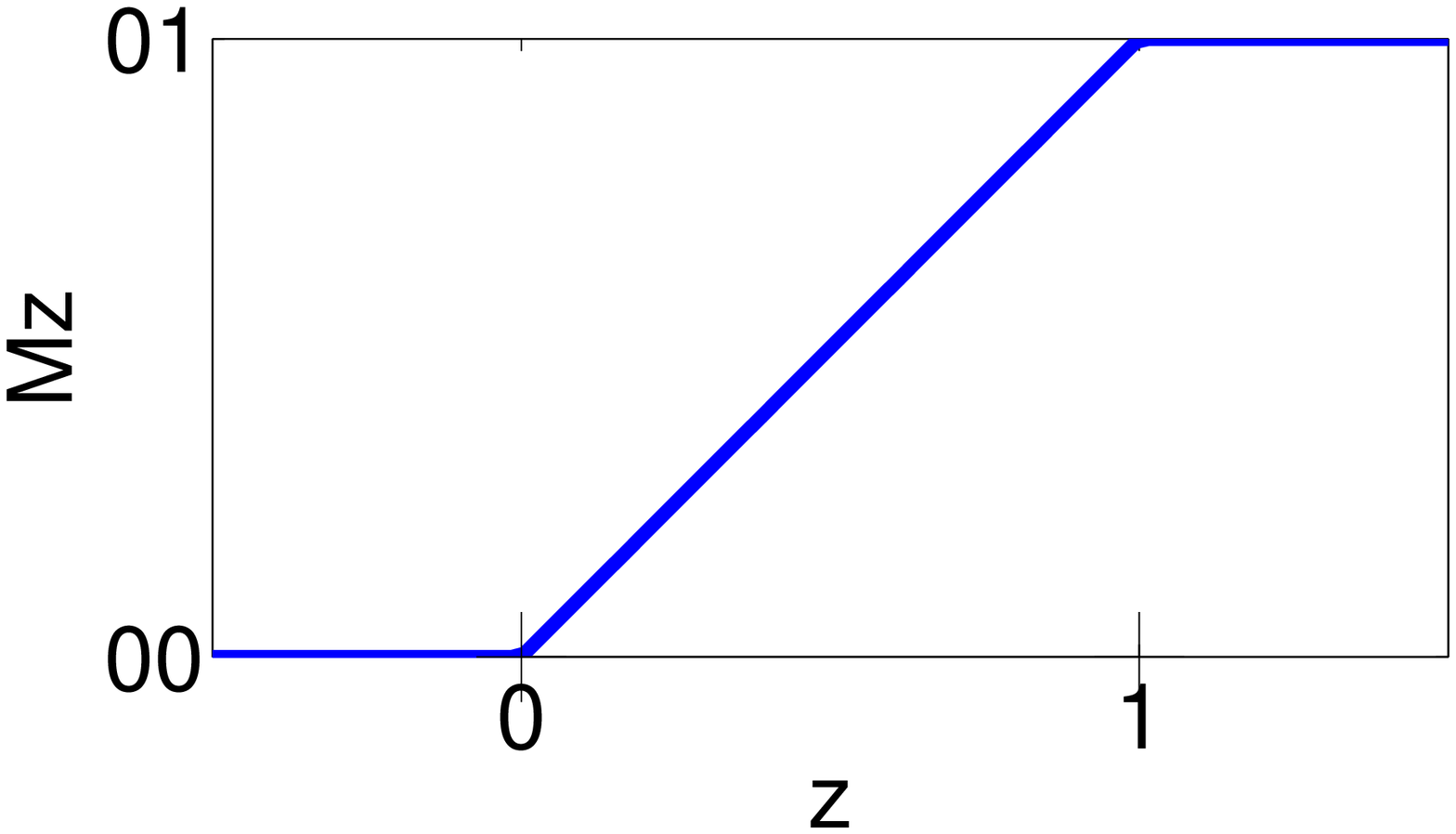}
  \end{tabular}
  \caption{\emph{Left}: three possible cases for the location of the solution for the scalar box-constrained problem~\eqref{e:scalar}. \emph{Right}: the solution $x^*_d = \text{median}(\overline{x}_d,l_d,u_d)$ is the median of $\overline{x}_d$, $l$ and $u$.}
  \label{f:median}
\end{figure}

Finally, the following result shows that, with axis-aligned trees, we obtain the same solution whether we use the $\ell_1$ norm or the $\ell_2$ norm or a linear combination of both.
\begin{cor}
  \label{th:L1_L2}
  In problem~\eqref{e:basic-counterfactual}, assume that each constraint depends on a single element of \x\ (not necessarily the same) and that the objective function is $E(\x;\overline{\x}) = \lambda_1 \norm{\x - \overline{\x}}_1 + \lambda_2 \norm{\x - \overline{\x}}^2_2$ with coefficients $\lambda_1,\lambda_2 \ge 0$. Then the solution of the problem is the same, regardless of the values of the coefficients, and it is given by applying the median formula of theorem~\ref{th:scalar} elementwise within each leaf and then picking the best leaf.
\end{cor}
However, note that if we use a different weight \emph{per feature}, e.g.\ $E(\x;\overline{\x}) = \sum^D_{d=1}{ \lambda_d (x_d - \overline{x}_d)^2 }$, then the optimal solution does depend on those weights: while it can still be computed elementwise within each leaf, which leaf is the best depends on the weights.

\subsection{Non-separable problems: oblique trees}
\label{s:algorithm:nonseparable}

With oblique trees, the root-leaf path constraints $\h_i(\x) \ge \0$ involve each a linear combination of multiple features, as shown in fig.~\ref{f:T2} (right). Hence the problem~\eqref{e:basic-counterfactual-leaf} over a leaf does not separate and cannot be solved elementwise over each feature. However, it can still be solved exactly and efficiently using LP or QP solvers.

Computationally, it is convenient to store in each leaf its root-leaf path constraints and possibly to preprocess them in order to make the subsequent optimization more efficient (as is customary with LP or QP solvers). For example, one can remove redundant constraints in the root-leaf path (e.g.\ $f_1$ is redundant for leaf 16 in fig.~\ref{f:T1}). Also, a good initialization for each QP is given by the source instance $\overline{\x}$.

\subsection{Useful cost or distance functions}
\label{s:algorithm:E}

The function $E(\x;\overline{\x})$ measures the cost of changing features in the source instance $\overline{\x}$, so it should satisfy $E(\overline{\x};\overline{\x}) = 0$ and $E(\x;\overline{\x}) > 0$ if $\x \neq \overline{\x}$ (or perhaps $E(\x;\overline{\x}) \ge 0$). An appropriate definition of $E$ is critical to find good counterfactual explanations, but such a definition depends on the application. For example, changing the amount of a loan is easier than changing the education level of a person. That said, a useful cost function can generally be written using a distance or a combination of distances, possibly weighted. Next, we give several generic distances that are convex and can be easily handled with decision trees. All of them have been used in earlier works, with the exception of the general quadratic distance, as far as we know.
\begin{itemize}
\item $\ell_2$ distance: $E(\x;\overline{\x}) = \norm{\x - \overline{\x}}^2_2$.
\item $\ell_1$ distance: $E(\x;\overline{\x}) = \norm{\x - \overline{\x}}_1$. This encourages few features to be changed, while the $\ell_2$ distance typically changes all features. \\
  To optimize a problem of the form ``$\min_{\bdelta}{ \norm{\bdelta}_1 }$ s.t.\ $\bdelta \in \calF$'' (where $\bdelta = \x - \overline{\x}$ and \calF\ is a polytope), we use the standard reformulation as a LP ``$\min_{\bdelta,\t}{ \1^T \t }$ s.t.\ $\bdelta \le \t,\ \bdelta \ge -\t,\ \bdelta \in \calF$''.
\item General quadratic distance.
  \begin{equation}
    \label{e:Q}
    E(\x;\overline{\x}) = \bdelta^T \Q \bdelta = \sum^D_{d,e=1}{ q_{de} \, \delta_d \, \delta_e}
  \end{equation}
  where we call $\bdelta = \x - \overline{\x}$ and \Q\ is a symmetric positive definite (or perhaps positive semidefinite) matrix, so $E$ is lower bounded. This means the distance is a sum of two types of costs:
  \begin{itemize}
  \item Singleton feature cost: $q_{dd} \, \delta^2_d \ge 0$. This is the cost of changing feature $d$ by $\delta_d$.
  \item Pair of features' cost: $q_{de} \, \delta_d \, \delta_e$. This is the cost of changing feature $d$ by $\delta_d$ and feature $e$ by $\delta_e$. We can have $q_{de} < 0$ or $> 0$ as long as \Q\ is positive (semi)definite.
  \end{itemize}
  The singleton cost is useful to represent differential costs depending on the feature (as in a weighted $\ell_2$ norm) while the pairwise cost is useful to represent costs that affect groups of features. For example, if \x\ are pixel values of a grayscale image, defining $q_{de}$ according to whether pixels $d$ and $e$ are neighboring can encourage changing local groups of pixels rather than arbitrary, isolated pixels. \\
  Making \Q\ positive semidefinite allows us to have zero cost for feature changes \bdelta\ satisfying $\Q \bdelta = \0$, which can be useful to represent invariances. For example, if \x\ are pixel grayscale values, having $\Q\1 = \0$ (where $\1 = (1,\dots,1)^T$) means that the cost of shifting all pixels by 1, i.e., globally changing the image intensity, is zero.
\item Finally, we can have combinations of all the above, such as $E(\x;\overline{\x}) = \norm{\x - \overline{\x}}_1 + \lambda \norm{\x - \overline{\x}}^2_2$.
\end{itemize}
We emphasize that in practice the $\ell_1$ or $\ell_2$ distances should be weighted (or equivalently each feature should be normalized). Such weights should be chosen by the user according to the range of variation of each feature and to its perceived cost.

\subsection{Useful constraints}
\label{s:algorithm:C}

The constraints $\cc(\x) = \0$ and $\dd(\x) \ge \0$ in problem~\eqref{e:basic-counterfactual} can be used to represent restrictions that must be obeyed for a counterfactual explanation to be reasonable. We consider the following:
\begin{itemize}
\item Constraints intrinsic to the problem: these typically represent natural equality constraints or lower and upper limits of each variable. We give some examples. Many variables are nonnegative, such as salary or cholesterol level. For a grayscale image each pixel should be in the interval [0,1] (black to white). For a color image, suitable intervals must be obeyed depending on the color space (RGB, LUV, etc.). The race of an individual cannot change from what it is. The age of an individual must be nonnegative and smaller than, say, 120 years. Further, if feature $d$ is the age of an individual, then we should constrain $x_d \ge \overline{x}_d$ to indicate that the given individual can get older but not younger.
\item Constraints desirable for a particular explanation: these are given by the user on a case-by-case basis. For example, a loan applicant may not want to change her marital status, and cannot increase her age by more than say a few months, even if either of those were possible and resulted in the loan being approved by the tree.
\item Categorical variables: because we handle them as continuous with a one-hot encoding, this introduces some constraints (see section~\ref{s:algorithm:categ}).
\item We can constrain \x\ to be in a discrete set of instances, such as the training set $x_1,\dots,\x_N$ of the tree. The optimization reduces to a simple search in the set: we evaluate the objective $E(\x;\overline{\x})$ on every instance $\x_n$ whose ground-truth class is the target class and satisfies the additional constraints $\cc(\x_n) = \0$, $\dd(\x_n) \ge \0$, and return the one with lowest objective. 
\end{itemize}

\subsection{Categorical variables}
\label{s:algorithm:categ}

Although many popular benchmarks and models in machine learning use only continuous variables, categorical variables are very important in practice, particularly in legal, financial or medical applications. And it is precisely in these human-related applications where counterfactual explanations might be most useful.

We handle categorical variables by encoding them as one-hot. That is, if an original categorical variable can take $C$ different categories, we encode it using $C$ dummy binary variables jointly constrained so that exactly one of them is 1 (for the corresponding category): $x_1,\dots,x_C \in \{0,1\}$ s.t.\ $\1^T \x = 1$.

During training with the TAO algorithm, we treat the dummy variables as if they were continuous and without the above constraints. This causes no problem because we only need to read the values of those variables; we do not need to update them.

When solving the counterfactual problem, we do modify those variables and so we need to respect the above constraints. This makes the problem a mixed-integer optimization, where some variables are continuous and others binary (the dummy variables). While these problems are NP-hard in general, in many practical cases we can expect to solve them exactly and quickly for two reasons: 1) categorical variables typically arise in low-dimensional problems and do not have many categories, so the total number of binary dummy variables is relatively small. And 2) modern mixed-integer optimization solvers, such as CPLEX or Gurobi \cite{Gurobi19a}, can solve relatively large problems exactly, and even larger ones approximately (providing a feasible result and a lower bound to the optimal objective) \cite{Bixby12a}.

Note that the simple approach of relaxing each integer constraint $x_c \in \{0,1\}$ to $x_c \ge 0$, solving a continuous optimization and rounding its result (i.e., picking the category $c$ with the largest $x_c$ value) has two drawbacks, which we have observed in our experiments: the result can be a poor approximation of the optimum; and, worse, applying the tree to it may not predict the target class (i.e., the rounded instance is infeasible).

Finally, with separable problems (section~\ref{s:algorithm:separable}), all $C$ dummy variables associated with an original categorical variable taking $C$ categories can be separated from all the other variables in the problem, but the $C$ dummy variables are optimized jointly. This is done simply by enumeration, i.e., trying each of the $C$ categories and picking the best.

\subsection{Extensions of the basic problem}
\label{s:algorithm:extensions}

The basic counterfactual problem~\eqref{e:basic-counterfactual} can be extended in several ways while remaining computationally easy:
\begin{itemize}
\item By their very nature, decision trees provide a simple way to offer the user not just a single solution for x but a \emph{diverse set of solutions}: we simply return the solutions of all the leaves having as label the target class label (assuming there are multiple leaves in that class), sorted in increasing cost $E$. Because the leaves' regions will usually lie far apart in instance space, their solutions will markedly differ in character.
\item We can define as target not a single class $y \in \{1,\dots,K\}$ but a (nonempty) subset of classes $\calY \subset \in \{1,\dots,K\} \setminus \overline{y}$:
  \begin{gather*}
    \min_{\x \in \bbR^D}{ E(\x;\overline{\x}) } \quad \text{s.t.} \quad T(\x) \in \calY,\ \cc(\x) = \0,\ \dd(\x) \ge \0 \Longrightarrow \\
    \min_{i \in \calL}{ \min_{\x \in \bbR^D}{ E(\x;\overline{\x}) } } \quad \text{s.t.} \quad y_i \subset \calY,\ \h_i(\x) \ge \0,\ \cc(\x) = \0,\ \dd(\x) \ge \0.
  \end{gather*}
  This is solved by enumeration over all leaves whose class is in \calY.
\item More generally, rather than a preferred class or class subset, we can consider all classes as feasible but with different costs. We can move the tree constraint to the objective, which becomes ``$\min_{\x}{ E(\x;\overline{\x}) + L(T(\x)) }$'', where $L(y) \ge 0$ is the cost for class $y \in \{1,\dots,K\}$. We can still solve over $y$ by enumeration over the leaves.
\item We may want to keep the counterfactual explanation away from class boundaries. The solution of problem~\eqref{e:basic-counterfactual} is always on the boundary of a leaf region of the target class, hence an infinitesimal perturbation will make it change class. This may not be a problem in some applications, for example ``barely passing'' could be enough to get a loan approved. However, sometimes we may want to find a counterfactual explanation that is safely classified as the target class. One simple way to avoid this is to shrink the leaf region away from its boundaries by shifting each leaf's constraints (box or polytope faces) by $\epsilon \ge 0$: $\h_i(\x) \ge \epsilon$. This provides a way to offer a range of counterfactual explanations for varying values of $\epsilon$. The problem remains convex and can be efficiently solved by solving first for $\epsilon = 0$ and then warm-starting the problem for a larger $\epsilon$ from the solution for the previous $\epsilon$ value (for each leaf).
\end{itemize}

\section{Experiments}
\label{s:expts}

Our algorithm is exact, as we have shown. It is also very efficient for both axis-aligned and oblique trees, even if there are categorical variables; solving a counterfactual explanation in our experiments takes usually milliseconds. We show 3 types of results: an example illustrating how the algorithm can be used interactively; summary results in several datasets, comparing with other algorithms; and results on MNIST digit images, which can be visualized.

\subsection{Illustrative example}
\label{s:illustrative}

\begin{table}[p]
  \centering
  \begin{tabular}{@{}c | c@{} | c@{} |  c@{} |  c@{}}
    \toprule 
    Feature &\caja{c}{c}{$\overline{x}$, source\\instance}&\caja{c}{c}{$\x^{*}_1$, no\\constraints}  & \caja{c}{c}{$\x^{*}_2$, some\\constraints} & \caja{c}{c}{$\x^{*}_3$, more\\constraints}  \\   
    \midrule	
    age  &  25 & $=$ & $=$ & $=$ \\
    workclass  &  Private & $=$ & $=$ & Federal-gov \\
    education  &  11th & $=$ & Assoc-voc & $=$\\
    marital-status & Never-married & $=$ & Married-AF-spouse & $=$ \\
    occupation  & Machine-op-inspect & $=$ & $=$ & Armed-Forces \\
    relationship  &  Own-child &$=$ & $=$ & $=$\\
    race  &  Black  & Asian-Pac-Islander & $=$ & $=$\\
    sex  &  Male  & $=$ & $=$ & $=$\\
    capital-gain  &  0 & $=$ & $=$ & 1\\
    capital-loss  &  0 & 1 & $=$ & 3\\
    hours-per-week &  40 & $=$ & 39 & $=$\\
    native-country  &  United-States  & Peru & $=$& $=$\\
    \hline
    income &  $<$\$50k & $\geq$\$50K & $\geq$\$50K & $\geq$\$50K\\
    \bottomrule
  \end{tabular}
  \caption{Example illustrating the construction of counterfactual instances with our exact algorithm for an oblique decision tree on the Adult dataset. We show the dataset features, source instance $\overline{\x}$ (of class ``$<$\$50k''), and 3 counterfactual instances (of class ``$\ge$\$50k'') with progressively more user constraints ($\x^*_1$, $\x^*_2$, $\x^*_3$). ``='' means the feature value is the same as in the source instance.}
  \label{t:AdultOblqiueExample}
\end{table}

We consider the UCI Adult dataset for binary classification, where each instance corresponds to a person (age, education, sex, etc., involving both continuous and categorical features), and the classes are whether or not the person makes over \$50k a year. We are given an oblique decision tree trained by TAO. We take the source instance in table~\ref{t:AdultOblqiueExample}, which is classified by the tree as ``below \$50k'', and seek counterfactual explanations $\x^*$ (in $\ell_2$ distance)  classified as ``above \$50k''. Without any constraints, 3 features change: race, capital-loss and native-country ($\x^*_1$). Changing a person's race and native-country is not possible, so we constrain race, native-country as well as sex not to change. This results in 3 other features changing: education, marital-status and hours-per-week ($\x^*_2$). Based on the user's preferences we then constrain education, marital-status and relationship not to change. Finally, this results in changing workclass, occupation, capital-gain and capital-loss ($\x^*_3$).

\subsection{Algorithm comparison over different datasets}
\label{s:comparison}

\begin{table}[p]
  \centering
  \begin{tabular}{@{}l|c|cc|cc|ccc@{}}
    \toprule
    & \% c &\multicolumn{4}{c|}{Our exact algorithm}&\multicolumn{3}{c}{training \& test set}\\
    &&ms&$\ell_2$&ms&$\ell_1$&ms &$\ell_2$& \% feasible \\
    \midrule
    &0& 40 & 0.63$\pm$0.48 &  580 & 2.85$\pm$1.40  & 40 & 53.50$\pm$17.24 & 100\\
    \raisebox{0pt}[0pt][0pt]{\rotatebox{90}{\makebox[0pt][c]{MNIST}}}
    &9& 40 & 0.63$\pm$0.48 & 580 & 2.85$\pm$1.40  & 40 & 53.50$\pm$17.24 & 100\\
    &47& 40 & 9.28$\pm$6.70 & 550 & 12.80$\pm$7.89 & --- & --- & 0 \\
    \midrule
    &0& 710 & 2.40$\pm$0.83& 600 & 2.40$\pm$0.83 &  70 & 3.4e4$\pm$1.2e5 & 100\\
    \raisebox{0pt}[0pt][0pt]{\rotatebox{90}{\makebox[0pt][c]{Adult}}}
    &7& 810 & 2.45$\pm$0.86 & 590 & 2.40$\pm$0.83 & 4300 & 1.6e7$\pm$5.2e7 & 100\\
    &14& 780 & 2.49$\pm$0.97& 570 & 2.50$\pm$0.97 & 2700 & 1.9e7$\pm$5.8e7 & 100\\	
    \midrule
    &0& 4 & 0.35$\pm$0.33&  2 & 1.14$\pm$0.48 & 0 & 0.86$\pm$0.49 & 100\\
    \raisebox{0pt}[0pt][0pt]{\rotatebox{90}{\makebox[0pt][c]{Breast}}}
    &11& 3 & 0.48$\pm$0.44&  2 & 1.14$\pm$0.48& 0 & 1.25$\pm$0.54 & 77\\
    &22& 3 & 0.52$\pm$0.48&  2 & 1.16$\pm$0.48& 0 & 1.61$\pm$0.89 & 17 \\	
    \midrule
    &0& 7 & 0.002$\pm$0.01& 12 & 0.060$\pm$0.06 &  0 & 0.060$\pm$0.06 & 100\\
    \raisebox{0pt}[0pt][0pt]{\rotatebox{90}{\makebox[0pt][c]{\small Spambase}}}
    &17& 6 & 0.002$\pm$0.01&  12 & 0.060$\pm$0.06 & 10 & 0.070$\pm$0.09 & 32\\
    &53& 7 & 0.003$\pm$0.01& 12 & 0.070$\pm$0.06  & 20 & 0.020$\pm$0.01 & 17 \\			
    \midrule
    &0& 18 & 0.02$\pm$0.01& 20 & 0.31$\pm$0.12 &  0 & 0.19$\pm$0.09 & 100\\
    \raisebox{0pt}[0pt][0pt]{\rotatebox{90}{\makebox[0pt][c]{ Letter}}}
    &25& 17 & 0.03$\pm$0.02& 20 & 0.31$\pm$0.12  & 0 & 0.36$\pm$0.19 & 19 \\
    &62& 16 & 0.20$\pm$0.84 &190 & 0.55$\pm$0.47  & --- & --- & 0\\
    \midrule
    &0& 50 & 3.97$\pm$3.50& 80 & 2.70$\pm$1.20 &  0 & 3.0e3$\pm$7.3e3 & 100\\
    \raisebox{0pt}[0pt][0pt]{\rotatebox{90}{\makebox[0pt][c]{ Credit}}}
    &7& 50 & 4.00$\pm$3.50 &  80 & 2.70$\pm$1.20 &  --- & --- & 0 \\
    &16& 50 & 4.25$\pm$5.80 & 80 & 2.80$\pm$1.30 &  --- & --- & 0\\
    \bottomrule	
  \end{tabular}
  \caption{Comparison of counterfactuals generated by our exact algorithm and by searching over the training \& test set, over a variety of datasets, using an oblique classification tree. We show: the percentage of features constrained in the source instance (\% c), average runtime per instance in milliseconds (ms), $\ell_2$ or $\ell_1$ distance (mean and standard deviation over 20 instances per class), and the percentage of feasible counterfactuals (for the search in the training \& test set only, since for our algorithm it is always 100\%).}
  \label{t:obliqueEval}
\end{table}

\begin{table}[t]
  \centering
  \begin{tabular}{@{}l|c|cc|cc|ccc@{}}
    \toprule
    & \% c &\multicolumn{4}{c|}{Our exact algorithm}&\multicolumn{3}{c}{Feature tweak \cite{Tolomei_19a}} \\
    &&ms&$\ell_2$ &ms &$\ell_1$ & ms &$\ell_2$& \% feasible \\
    \midrule
    &0& 110 & 0.04$\pm$0.11 & 110 & 0.16$\pm$0.26 & 190 & 1.48$\pm$1.09 & 100\\
    \raisebox{0pt}[0pt][0pt]{\rotatebox{90}{\makebox[0pt][c]{MNIST}}}
    &9& 110 & 0.04$\pm$0.11 & 110 & 0.16$\pm$0.26  & ---& ---& ---\\
    &47& 120 & 5.83$\pm$3.27 & 120 & 13.16$\pm$3.65 &  ---& ---& ---\\
    \midrule
    &0& 130 & 2.05 $\pm$0.31& 130 & 2.05$\pm$0.31& 50 & 1.00$\pm$0.00 & 28\\
    \raisebox{0pt}[0pt][0pt]{\rotatebox{90}{\makebox[0pt][c]{Adult}}}
    &7& 130 & 2.07$\pm$0.34 & 130 & 2.07$\pm$0.34 & ---& ---& ---\\
    &14& 140 & 26.10$\pm$14.97 & 140 & 2.85$\pm$4.54 & ---& ---& ---\\
    \midrule
    &0& 0 & 0.13$\pm$0.17 & 0 & 0.42$\pm$0.28& 0 & 0.36$\pm$0.19 & 100\\
    \raisebox{0pt}[0pt][0pt]{\rotatebox{90}{\makebox[0pt][c]{Breast}}}
    &11& 0 & 0.19 $\pm$0.25& 0 & 0.47$\pm$0.33  & ---& ---& ---\\
    &22& 0 & 0.22$\pm$0.29 & 0 & 0.48 $\pm$0.35 & ---& ---& ---\\
    \midrule
    &0& 03 & 1.7e---5$\pm$0.0 & 30 & 0.003$\pm$0.002 & 0 & 0.005$\pm$0.009 & 100\\
    \raisebox{0pt}[0pt][0pt]{\rotatebox{90}{\makebox[0pt][c]{\small Spambase}}}
    &17& 30 & 1.7e---5$\pm$0.0 & 30 & 0.004$\pm$0.002  & ---& ---& ---\\
    &53& 40 & 4.5e---5$\pm$0.0 & 40 & 0.005$\pm$0.004 & ---& ---& ---\\
    \midrule
    &0& 50 & 0.014$\pm$0.01 & 50 & 0.16$\pm$0.08 & 30 & 0.22$\pm$0.13 & 100\\
    \raisebox{0pt}[0pt][0pt]{\rotatebox{90}{\makebox[0pt][c]{ Letter}}}
    &25& 40 & 0.016$\pm$0.02 & 50 & 0.17$\pm$0.09 & --- & --- & ---\\
    &62& 40 & 0.058$\pm$0.05 & 60 & 0.28$\pm$0.02 & --- & --- & ---\\
    \midrule
    &0& 40 & 2.60$\pm$1.01 & 40 & 2.60$\pm$1.01 & 0 & 87.88$\pm$13.20 & 22\\
    \raisebox{0pt}[0pt][0pt]{\rotatebox{90}{\makebox[0pt][c]{ Credit}}}
    &7& 40 & 2.60$\pm$1.01 & 40 & 2.60$\pm$1.01 & --- & --- & ---\\
    &16& 40 & 26.50$\pm$50.2 & 40 & 4.87$\pm$4.64 & --- & --- & ---\\
    \bottomrule		
  \end{tabular}
  \caption{Like table~\ref{t:obliqueEval} but for an axis-aligned classification tree. We also show the results for the Feature Tweaking method of \cite{Tolomei_19a} (which does not apply to problems with constraints, marked ``---'').}
  \label{t:AxisAlignedeval}
\end{table}

We now run our algorithm on several source instances from several datasets (of different numbers of features $D$ and classes $K$, and with continuous and/or categorical features), on both an oblique tree trained by TAO (table~\ref{t:obliqueEval}) and an axis-aligned tree trained by CART (table~\ref{t:AxisAlignedeval}). For each dataset we randomly select 20 source instances of each class from the test instances and generate a counterfactual explanation for them (to target a different class), using the $\ell_2$ or $\ell_1$ distance. As in the illustrative example, we try 3 levels of constraints: no constraints, some constraints and even more constraints (picked at random or manually); the tables show the percentage of features constrained in each case.

In the oblique tree, we compare with searching only over those training \& test set instances (as in the What-If tool \cite{Wexler_20a}) which are classified as the target class by the tree, using the $\ell_2$ distance. We label this as ``training \& test set'' in the tables and figures. (Note that the tree is not perfect and may misclassify an instance, as happens in the first and last rows of fig.~\ref{f:MNISTCFWithQ}.) Clearly, this produces counterfactual instances with far larger distances and fails to find a feasible counterfactual instance (i.e., of the target class) if too many constraints are applied.

In the axis-aligned tree, we compare with the Feature Tweaking algorithm \cite{Tolomei_19a} in the $\ell_2$ distance, by running the authors' implementation (note this algorithm does not apply to oblique trees). Since it handles only continuous features, to use categorical ones we encode them as one-hot, solve as if they were continuous and round them at the end. We clearly see that Feature Tweaking is not exact for binary axis-aligned trees, contradicting the claim in \cite{Tolomei_19a}. This is shown by the larger distances and by the failure to find a feasible counterfactual instance if too many constraints are applied. Our algorithm is indeed exact for both oblique and axis-aligned trees and returns a feasible, minimal-distance counterfactual instance every time.

The runtime of our algorithm is a few milliseconds per instance, even in relatively high-dimensional cases such as the MNIST dataset ($D = 784$), or involving around 10 categorical features translating into almost 100 binary dummy variables as in the Adult dataset.

\subsection{MNIST dataset with oblique trees}

Figure~\ref{f:MNISTCFWithQ} shows results using an oblique tree trained with TAO for MNIST digit images, so that we can visualize the result of different distances. We constrain each pixel to be in [0,1] so that the counterfactual instance is a valid grayscale image. For both the $\ell_1$ and $\ell_2$ distances, the counterfactual instance is visually barely distinguishable from the source instance. As is well known, the $\ell_1$ distance results in few pixels changing but by a large amount, while the $\ell_2$ distance results in most pixels changing but by a small amount, and in both cases the pixel locations are arbitrary.

We then tried using a general quadratic distance. We constructed a positive definite matrix \Q\ of 784 $\times$ 784 having 1s in the diagonal and a value of $-\frac{1}{4}$ corresponding to neighboring pixels (up, down, left, right). We also constrained each pixel to obey $\x \ge \overline{\x}$ (add ink only) or $\x \le \overline{\x}$ (erase ink only). The resulting counterfactual images clearly show the changes occur on local groups of pixels.

In all those cases, the resulting counterfactual instances can be regarded as adversarial examples, in that they are visually hard to tell apart from the source instance, and not representative of the target class. The realism of the counterfactual instance can be improved by searching only over the training \& test set instances ($\ell_2$ distance). In this case, the counterfactual image is clearly recognizable as being from the target class, but the distance is far larger. Finding realistic counterfactual instances for images is a difficult, open problem \cite{Dhuran_18a,VanloovKlaise19a}.

\newcommand{\mysize}{0.086}
\begin{figure*}[p]
  \centering
  \begin{tabular}{@{}c@{\hspace{0.5ex}}c@{\hspace{1ex}}c@{}c@{\hspace{1ex}}c@{}c@{\hspace{1ex}}c@{}c@{\hspace{1ex}}c@{}c@{\hspace{1ex}}c@{}c@{}}
    & &  & &  & &   \multicolumn{2}{c@{}}{\underline{\makebox[0.14\linewidth][c]{only add ink}}} &  \multicolumn{2}{c@{}}{\underline{\makebox[0.14\linewidth][c]{only erase ink}}}&\multicolumn{2}{c@{}}{\underline{\makebox[0.14\linewidth][c]{training \& test set}}}\\ 
    &{$\overline{\x}$} & {$\ell_1(1.83)$} & {$\x^{*}-\overline{\x}$}& {$\ell_2(0.34)$} & {$\x^{*}-\overline{\x}$} & $\Q(0.06)$ &$\x^{*}-\overline{\x}$& $\Q(0.66)$ &$\x^{*}-\overline{\x}$& {$\ell_2(40.4)$} & {$\x^{*}-\overline{\x}$} \\ 
    \rotatebox{90}{\hspace*{2ex}$2 \rightarrow 0$}&
    \includegraphics*[width=\mysize\linewidth]{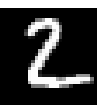}&
    \includegraphics*[width=\mysize\linewidth]{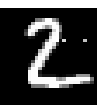}&
    \includegraphics*[width=\mysize\linewidth]{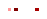}&
    \includegraphics*[width=\mysize\linewidth]{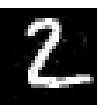}&
    \includegraphics*[width=\mysize\linewidth]{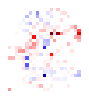}& 
    \includegraphics*[width=\mysize\linewidth]{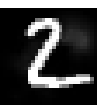}&
    \includegraphics*[width=\mysize\linewidth]{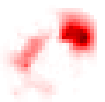}&
    \includegraphics*[width=\mysize\linewidth]{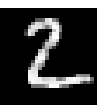}&
    \includegraphics*[width=\mysize\linewidth]{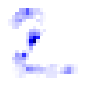}&
    \includegraphics*[width=\mysize\linewidth]{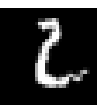}&
    \includegraphics*[width=\mysize\linewidth]{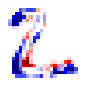} \\[1ex]
    &{$\overline{\x}$} & {$\ell_1(2.01)$} & {$\x^{*}-\overline{\x}$}& {$\ell_2(0.2)$} & {$\x^{*}-\overline{\x}$}& $\Q(0.04)$ &$\x^{*}-\overline{\x}$ & $\Q(0.45)$ &$\x^{*}-\overline{\x}$ & {$\ell_2(55.6)$} & {$\x^{*}-\overline{\x}$} \\ 
    \rotatebox{90}{\hspace*{2ex}$7 \rightarrow 1$}&
    \includegraphics*[width=\mysize\linewidth]{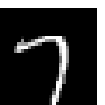}&
    \includegraphics*[width=\mysize\linewidth]{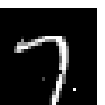}&
    \includegraphics*[width=\mysize\linewidth]{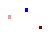}&
    \includegraphics*[width=\mysize\linewidth]{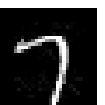}&
    \includegraphics*[width=\mysize\linewidth]{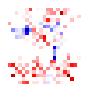}& 
    \includegraphics*[width=\mysize\linewidth]{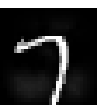}&
    \includegraphics*[width=\mysize\linewidth]{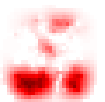}&
    \includegraphics*[width=\mysize\linewidth]{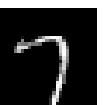}&
    \includegraphics*[width=\mysize\linewidth]{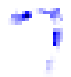}&
    \includegraphics*[width=\mysize\linewidth]{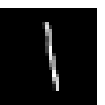}&
    \includegraphics*[width=\mysize\linewidth]{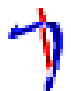} \\[1ex]
    &{$\overline{\x}$} & {$\ell_1(4.24)$} & {$\x^{*}-\overline{\x}$}& {$\ell_2(1.00)$} & {$\x^{*}-\overline{\x}$} & $\Q(0.27)$ &$\x^{*}-\overline{\x}$ & $\Q(1.11)$ &$\x^{*}-\overline{\x}$ & {$\ell_2(84.5)$} & {$\x^{*}-\overline{\x}$}\\ 
    \rotatebox{90}{\hspace*{2ex}$6 \rightarrow 3$}&
    \includegraphics*[width=\mysize\linewidth]{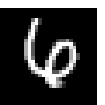}&
    \includegraphics*[width=\mysize\linewidth]{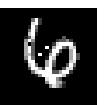}&
    \includegraphics*[width=\mysize\linewidth]{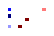}&
    \includegraphics*[width=\mysize\linewidth]{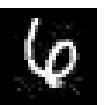}&
    \includegraphics*[width=\mysize\linewidth]{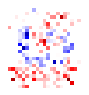}& 
    \includegraphics*[width=\mysize\linewidth]{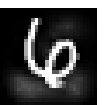}&
    \includegraphics*[width=\mysize\linewidth]{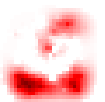}&
    \includegraphics*[width=\mysize\linewidth]{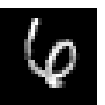}&
    \includegraphics*[width=\mysize\linewidth]{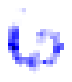}&
    \includegraphics*[width=\mysize\linewidth]{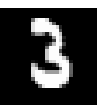}&
    \includegraphics*[width=\mysize\linewidth]{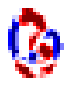} \\[1ex]
    &{$\overline{\x}$} & {$\ell_1(4.01)$} & {$\x^{*}-\overline{\x}$}& {$\ell_2(0.95)$} & {$\x^{*}-\overline{\x}$} & $\Q(0.24)$ &$\x^{*}-\overline{\x}$& $\Q(2.02)$ &{$\x^{*}-\overline{\x}$}& {$\ell_2(51.1)$} & {$\x^{*}-\overline{\x}$} \\ 
    \rotatebox{90}{\hspace*{2ex}$5 \rightarrow 4$}&
    \includegraphics*[width=\mysize\linewidth]{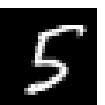}&
    \includegraphics*[width=\mysize\linewidth]{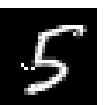}&
    \includegraphics*[width=\mysize\linewidth]{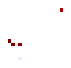}&
    \includegraphics*[width=\mysize\linewidth]{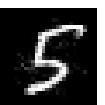}&
    \includegraphics*[width=\mysize\linewidth]{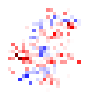}& 
    \includegraphics*[width=\mysize\linewidth]{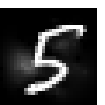}&
    \includegraphics*[width=\mysize\linewidth]{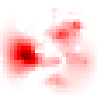}&
    \includegraphics*[width=\mysize\linewidth]{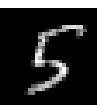}&
    \includegraphics*[width=\mysize\linewidth]{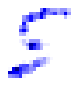}&
    \includegraphics*[width=\mysize\linewidth]{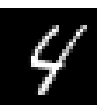}&
    \includegraphics*[width=\mysize\linewidth]{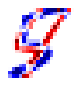} \\[1ex]
    &{$\overline{\x}$} & {$\ell_1(2.78)$} & {$\x^{*}-\overline{\x}$}& {$\ell_2(0.49)$} & {$\x^{*}-\overline{\x}$} & $\Q(0.07)$ &$\x^{*}-\overline{\x}$ & $\Q(2.55)$ &$\x^{*}-\overline{\x}$& {$\ell_2(49.1)$} & {$\x^{*}-\overline{\x}$} \\ 
    \rotatebox{90}{\hspace*{2ex}$3 \rightarrow 9$}&
    \includegraphics*[width=\mysize\linewidth]{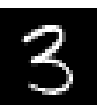}&
    \includegraphics*[width=\mysize\linewidth]{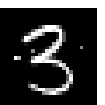}&
    \includegraphics*[width=\mysize\linewidth]{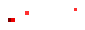}&
    \includegraphics*[width=\mysize\linewidth]{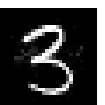}&
    \includegraphics*[width=\mysize\linewidth]{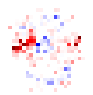}& 
    \includegraphics*[width=\mysize\linewidth]{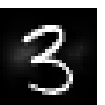}&
    \includegraphics*[width=\mysize\linewidth]{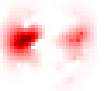}&
    \includegraphics*[width=\mysize\linewidth]{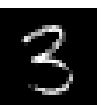}&
    \includegraphics*[width=\mysize\linewidth]{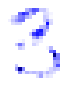}&
    \includegraphics*[width=\mysize\linewidth]{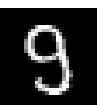}&
    \includegraphics*[width=\mysize\linewidth]{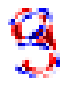} \\[1ex]
    &{$\overline{\x}$} & {$\ell_1(2.73)$} & {$\x^{*}-\overline{\x}$}& {$\ell_2(0.55)$} & {$\x^{*}-\overline{\x}$}& $\Q(0.23)$ &{$\x^{*}-\overline{\x}$} & $\Q(1.14)$ &$\x^{*}-\overline{\x}$& {$\ell_2(74.03)$} & {$\x^{*}-\overline{\x}$} \\
    \rotatebox{90}{\hspace*{2ex}$2 \rightarrow 7$}&
    \includegraphics*[width=\mysize\linewidth]{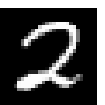}&
    \includegraphics*[width=\mysize\linewidth]{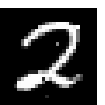}&
    \includegraphics*[width=\mysize\linewidth]{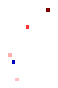}&
    \includegraphics*[width=\mysize\linewidth]{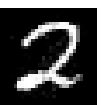}&
    \includegraphics*[width=\mysize\linewidth]{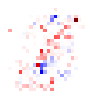}& 
    \includegraphics*[width=\mysize\linewidth]{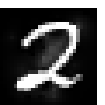}&
    \includegraphics*[width=\mysize\linewidth]{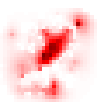}&
    \includegraphics*[width=\mysize\linewidth]{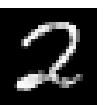}&
    \includegraphics*[width=\mysize\linewidth]{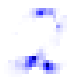}&
    \includegraphics*[width=\mysize\linewidth]{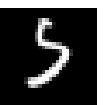}&
    \includegraphics*[width=\mysize\linewidth]{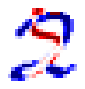} \\[1ex]
    &{$\overline{\x}$} & {$\ell_1(3.07)$} & {$\x^{*}-\overline{\x}$}& {$\ell_2(0.75)$} & {$\x^{*}-\overline{\x}$} & $\Q(0.14)$ &$\x^{*}-\overline{\x}$ & $\Q(2.66)$ &$\x^{*}-\overline{\x}$& {$\ell_2(57.25)$} & {$\x^{*}-\overline{\x}$}\\ 
    \rotatebox{90}{\hspace*{2ex}$7 \rightarrow 0$}&
    \includegraphics*[width=\mysize\linewidth]{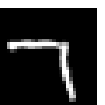}&
    \includegraphics*[width=\mysize\linewidth]{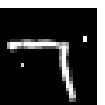}&
    \includegraphics*[width=\mysize\linewidth]{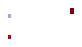}&
    \includegraphics*[width=\mysize\linewidth]{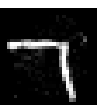}&
    \includegraphics*[width=\mysize\linewidth]{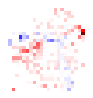}& 
    \includegraphics*[width=\mysize\linewidth]{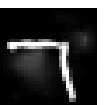}&
    \includegraphics*[width=\mysize\linewidth]{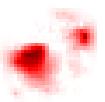}&
    \includegraphics*[width=\mysize\linewidth]{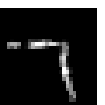}&
    \includegraphics*[width=\mysize\linewidth]{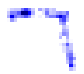}&
    \includegraphics*[width=\mysize\linewidth]{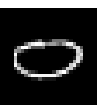}&
    \includegraphics*[width=\mysize\linewidth]{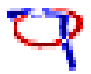} \\[1ex]
    &{$\overline{\x}$} & {$\ell_1(4.9)$} & {$\x^{*}-\overline{\x}$}& {$\ell_2(0.91)$} & {$\x^{*}-\overline{\x}$} & $\Q(0.39)$ &$\x^{*}-\overline{\x}$ & $\Q(1.69)$ &$\x^{*}-\overline{\x}$& {$\ell_2(81.73)$} & {$\x^{*}-\overline{\x}$}\\
    \rotatebox{90}{\hspace*{2ex}$2 \rightarrow 6$}&
    \includegraphics*[width=\mysize\linewidth]{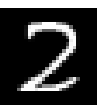}&
    \includegraphics*[width=\mysize\linewidth]{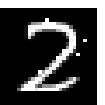}&
    \includegraphics*[width=\mysize\linewidth]{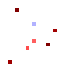}&
    \includegraphics*[width=\mysize\linewidth]{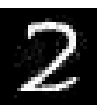}&
    \includegraphics*[width=\mysize\linewidth]{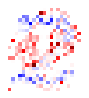}& 
    \includegraphics*[width=\mysize\linewidth]{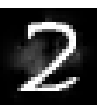}&
    \includegraphics*[width=\mysize\linewidth]{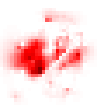}&
    \includegraphics*[width=\mysize\linewidth]{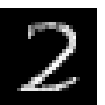}&
    \includegraphics*[width=\mysize\linewidth]{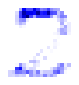}&
    \includegraphics*[width=\mysize\linewidth]{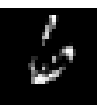}&
    \includegraphics*[width=\mysize\linewidth]{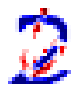}\\[1ex]
    &{$\overline{\x}$} & {$\ell_1(3.73)$} & {$\x^{*}-\overline{\x}$}& {$\ell_2(0.57)$} & {$\x^{*}-\overline{\x}$} & $\Q(0.12)$ &$\x^{*}-\overline{\x}$ & $\Q(0.81)$ &{$\x^{*}-\overline{\x}$}& {$\ell_2(49.35)$} & {$\x^{*}-\overline{\x}$}\\
    \rotatebox{90}{\hspace*{2ex}$6 \rightarrow 5$}&
    \includegraphics*[width=\mysize\linewidth]{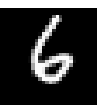}&
    \includegraphics*[width=\mysize\linewidth]{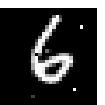}&
    \includegraphics*[width=\mysize\linewidth]{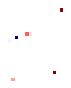}&
    \includegraphics*[width=\mysize\linewidth]{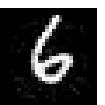}&
    \includegraphics*[width=\mysize\linewidth]{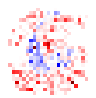}& 
    \includegraphics*[width=\mysize\linewidth]{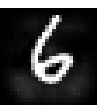}&
    \includegraphics*[width=\mysize\linewidth]{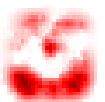}&
    \includegraphics*[width=\mysize\linewidth]{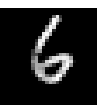}&
    \includegraphics*[width=\mysize\linewidth]{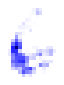}&
    \includegraphics*[width=\mysize\linewidth]{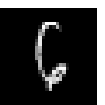}&
    \includegraphics*[width=\mysize\linewidth]{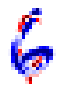} \\[1ex]
  \end{tabular}
  \caption{MNIST handwritten digit image counterfactuals using an oblique classification tree. Each row is a different source instance and target class (indicated at left, e.g.\ ``2 $\to$ 0'' means from class 2 to class 0). For each distance, we show the resulting counterfactual instance $\x^*$ (as a grayscale image) and its difference with the source instance $\x^* - \overline{\x}$ (red: positive, blue: negative, white: zero; rescaled within each image to [-1,+1]). We report the value of the distance in parenthesis. Zoom in to see details.}
  \label{f:MNISTCFWithQ}
\end{figure*}

\section{Conclusion}
\label{s:concl}

Classification trees are very important in applications such as business, law and medicine, where counterfactual explanations are of particular relevance. We have given an exact, efficient algorithm to compute counterfactual explanations for axis-aligned and oblique trees in multiclass problems, with different distances and constraints, and applicable to both continuous and categorical features. The algorithm is fast enough to allow interactive use. It should be possible to extend it to other cases, such as softmax classifier leaves (rather than constant-label leaves) and regression trees.

\section{Acknowledgments}

Work partially supported by NSF award IIS--2007147.

\appendix

\section{Details about the experiments}

\subsection{Dataset information}
\label{s:dataInfo}

In this section we describe the datasets used, in detail. All datasets are from UCI~\citep{Zhang_17c} except MNIST.

\begin{description}
\item[Adult] It is a dataset with mixed type attributes. The prediction task is to determine whether a person makes over 50K a year. There are 12 attributes, out of which 4 are continuous, and the rest are categorical. In all our experiments we convert each categorical attribute to one-hot encoding attribute. Thus each instance has 102 attributes. There are 30\,162 training instances, and separate 15\,062 test instances. In table~\ref{t:AdultDataInfo}, we explain the attributes in detail.
\item[Breast-Cancer] The task is to classify whether the cancer is malignant or benign. There are 699 instances and each instance has 9 real-valued attributes. Since there is no separate test dataset, we randomly divide the entire data into training (80\%) and test (20\%). 
\item[Spambase] This dataset consists of a collection of emails, and the task is to create a spam-filter that can tell whether an email is a spam or not. There are 4\,601 instances and each instance has 56 real-valued attributes. Since there is no separate test dataset, we randomly divide the entire data into training (80\%) and test (20\%).
\item[Letter] The objective of this dataset is to classify 26 capital letters in the English alphabet. It has separate 5\,000 test instances along with 15\,000 training instances. Each instance has 16 real-valued attributes. The character images were based on 20 different fonts and each letter within these 20 fonts was randomly distorted to produce a file of 20\,000 unique stimuli. Each stimulus was converted into 16 primitive numerical attributes (statistical moments and edge counts) which were then scaled to fit into a range of integer values from 0 through 15.
\item[German Credit] Similar to the Adult dataset, this dataset also has mixed type attributes. The prediction task is to determine whether an applicant is considered a Good or a Bad credit risk for 1\,000 loan applicants. There are 20 attributes, out of which 7 are continuous and rest are categorical. Similarly to the Adult dataset, we convert each categorical attribute to one-hot encoding attribute. Thus each instance of the dataset has 61 attributes. Unlike in the Adult dataset, no separate test set is available, so we randomly divide the 1\,000 instances into training (80\%) and test (20\%). 
\item[MNIST] The dataset consists of grayscale images of handwritten digits and the task is to classify them as 0 to 9. There are 60\,000 training images and 10\,000 test images. Each image is of size $28 \times 28$ with gray scales in [0,1].
\end{description}

\begin{table}[t]
  \centering
  \begin{tabular}{@{}cccc@{}}
    \toprule 
    Dataset & $D$ & Feature type & $K$ \\
    \midrule
    MNIST & 784 & Continuous & 10 \\
    Adult  & 102 & Continuous and categorical & 2\\
    Breast-Cancer& 9 & Continuous & 2\\
    Spambase & 57 & Continuous & 2\\
    Letter & 16 & Continuous & 26 \\
    German Credit & 61 & Continuous and categorical & 2\\
    \bottomrule
  \end{tabular}
  \caption{Short description of the datasets used in the paper.}
  \label{t:DataShortInfo}
\end{table}

\begin{table}[p]
  \centering
  \begin{tabular}{@{}cc@{\hspace{3ex}}c@{}}
    \toprule 
    Feature name & Feature type & Explanation \\
    \midrule
    age  &  Continuous & range: 17 to 90  \\ \\
    workclass  &  Categorical & \caja[0.8]{c}{c}{7 categories\\ Private, Self-emp-not-inc, Self-emp-inc,\\ Federal-gov, Local-gov, State-gov, \\Without-pay, Never-worked} \\ \\
    education  &  Categorical & \caja[0.8]{c}{c}{17 categories\\ Bachelors, Some-college,\\ HS-grad, Prof-school, \\ Assoc-acdm, Assoc-voc, \\ 11th, 9th, 7th-8th, 12th, \\ Masters, 1st-4th, 10th,\\ Doctorate, 5th-6th, Preschool.} \\ \\
    marital-status  &  Categorical & \caja[0.8]{c}{c}{7 categories\\ Married-civ-spouse, Divorced,\\ Widowed, Separated, \\ Never-married,\\ Married-spouse-absent\\ Married-AF-spouse.} \\ \\
    occupation  &  Categorical & \caja[0.8]{c}{c}{14 categories\\ Tech-support, Craft-repair,\\ Other-service, Sales, \\ Exec-managerial,\\ Prof-specialty, \\ Handlers-cleaners,\\ Machine-op-inspct,\\ Adm-clerical,\\Farming-fishing,\\ Transport-moving,\\ Priv-house-serv,\\ Protective-serv,\\ Armed-Forces.} \\ \\
    relationship  &  Categorical & \caja[0.8]{c}{c}{6 categories\\ Wife, Own-child, \\ Husband, Not-in-family,\\ Other-relative, Unmarried.} \\ \\
    race  &  Categorical & \caja[0.8]{c}{c}{5 categories\\ White, Asian-Pac-Islander, \\ Amer-Indian-Eskimo,\\ Other, Black.}\\ \\
    sex  &  Categorical & \caja[0.8]{c}{c}{2 categories\\  Female and  Male.} \\ \\
    capital-gain & Continuous & range: 0 to 99999 \\ \\
    capital-loss & Continuous & range: 0 to 4356 \\ \\
    hours-per-week & Continuous & range: 1 to 99  \\ \\
    native-country & Categorical &  41 countries\\ \\
    \hline
    income (classes) & Categorical &\caja[0.8]{c}{c}{ 1) less than 50K \\ 2) greater than equal to 50k}\\
    \bottomrule
  \end{tabular}
  \caption{Feature explanation for the Adult dataset.}
  \label{t:AdultDataInfo}
\end{table}

\subsection{Details about the constraints}

Below we describe the constraints used in tables~\ref{t:obliqueEval}--\ref{t:AxisAlignedeval}.
\begin{description}
\item[Adult] For the second row, we constrain ``race'' and ``sex'' to not change. Although they are two attributes, as a one-hot encoding these two attributes form 7 (7\%) attributes. For the third row along with ``race'' and ``sex'' we also fix the ``marital-status''. Thus in total 14 (14\%) attributes are fixed. 
\item[Breast Cancer] For the second row, we randomly select one (11\%) attribute index and for a given input we fix the attribute value at that index. For the third row we again select one more attribute index at random and fix the attribute value at those indexes (including one from the second row too), so in total 2 (22\%) attributes are fixed.
\item[Spambase] For the second row, we randomly select 10 (17\%) attribute indexes and for a given input we fix the attribute values at those indexes. For the third row we select 20 more attribute indexes at random and fix the attribute value at those indexes (including one from the second row too), so in total 30 (53\%) attributes are fixed.
\item[Letter] For the second row, we randomly select 4 (25\%) attribute indexes and for a given input we fix the attribute values at those indexes. For the third row we select 6 more attribute indexes at random and fix the attribute value at those indexes (including one from the second row too), so in total 10 (62\%) attributes are fixed.
\item[German Credit] For the second row, we constrain ``sex'' and ``status'' to not change. Although they are two attributes, as a one-hot encoding these two attribute form 4 (7\%) attributes. For the third row along with ``sex'' and ``status'' we also fix ``Credit history''. Thus in total 9 (14\%) attributes are fixed.
\item[MNIST] For the second row, we fix those pixels which are always zero in the entire training dataset. There are 69 (9\%) such pixels, we constrain them to be 0. For the third row along with 69 pixels we randomly select 200 more pixels and fix them too, so in total 269 (47\%) attributes are constrained.
\end{description}

\subsection{Details about the algorithms}

All our algorithms and experiments are implemented in Python. We train CART and Random Forest using Scikit-learn (version 0.22). For CART we train by first letting the tree grow full, and then prune it back with optimal pruning parameter. We train Random Forest using the default parameters in Scikit-learn. We train oblique trees using TAO (ran for 60 iterations at most), also implemented in Python. For oblique trees, we initialize the tree with random parameters (weights and biases) and a user-set depth (mentioned in table~\ref{t:comparison}). For each tree we use same $\ell_1$ parameter which is equal to 10. For solving the quadratic and linear programs, we use Gurobi Python interface (version 9.0). We use Gurobi's \texttt{Mvar} type variables to implement all our problems.

\begin{table*}[t]
  \centering
  \begin{tabular}{@{}c|ccc|ccc|cc@{}}
    \toprule 
    &\multicolumn{3}{c|}{CART} & \multicolumn{3}{c|}{Oblique (TAO)} & \multicolumn{2}{c@{}}{Random Forest} \\[1ex]
    Dataset & \caja[0.8]{c}{c}{Train \\ error}  & \caja[0.8]{c}{c}{Test \\ error} & Depth & \caja[0.8]{c}{c}{Train \\ error} & \caja[0.8]{c}{c}{Test \\ error} & Depth & \caja[0.8]{c}{c}{Train \\ error} & \caja[0.8]{c}{c}{Test \\ error} \\
    \midrule
    MNIST & 4.2 & 11.9 & 19 & 1.4 & 5.2 & 9 & 0 & 3.1\\
    Adult  & 12.6  & 14.7 & 12 &13.9 & 14.6 &12 & 13.2 & 14.1\\
    Breast-Cancer& 1.4 & 2.6 & 4& 1.2 & 2.1 & 3 & 0.01 & 2.1\\
    Spambase & 3.1 & 7.6 & 10 & 4.5 & 4.8 & 5 & 0 & 4.2\\
    Letter & 1.2  & 12.1 & 25 & 3.8 & 7.7 & 12 & 0.0 & 3.7 \\
    German Credit & 21.9 & 23.5 & 7 & 17.5 & 18.5 & 8 & 0.0 & 21.9\\
    \bottomrule
  \end{tabular}
  \caption{Test/training error of oblique trees (TAO), axis-aligned trees (CART) and Random Forest over various datasets.}
  \label{t:comparison}
\end{table*}

\bibliographystyle{abbrvnat}

\end{document}